\documentclass[11pt, a4paper, logo]{googledeepmind}

\pdftrailerid{redacted}

\newcommand{\yunhaocomment}[1]{\textcolor{blue}{[Yunhao: #1]}}
\newcommand{\markcomment}[1]{\textcolor{red}{[Mark: #1]}}
\newcommand{\zeyucomment}[1]{\textcolor{orange}{[Zeyu: #1]}}
\newcommand{\todom}[1]{\textcolor{green}{[Michal: #1]}}
\newcommand{\danielcomment}[1]{\textbf{\textcolor{teal}{[DanielG: #1]}}}
\newcommand{\bernacomment}[1]{\textbf{\textcolor{olive}{[Berna: #1]}}}

\makeatletter
\renewcommand\bibentry[1]{\nocite{#1}{\frenchspacing\@nameuse{BR@r@#1\@extra@b@citeb}}}
\makeatother

\usepackage{parskip}
\usepackage{amsmath, amsfonts, bm, dsfont}
\usepackage{kantlipsum, lipsum}
\usepackage{dsfont}
\usepackage{gdm-colors}

\usepackage{thm-restate}

\usepackage[authoryear, sort&compress, round]{natbib}

\usepackage{amsmath,amsfonts,bm}

\newtheoremstyle{definition}
{3pt} 
{3pt} 
{} 
{} 
{\bfseries} 
{.} 
{.5em} 
{} 

\theoremstyle{definition}

\usepackage{hyperref}
\usepackage{url}
\usepackage{amsthm}
\usepackage{subfigure}
\usepackage{graphicx}

\usepackage{bbm}

\usepackage{enumitem}
\setlist[itemize]{leftmargin=0.5cm}

\usepackage{xcolor}

\title{Generalized Preference Optimization:
A Unified Approach to Offline Alignment}

\reportnumber{}

\renewcommand{\today}{}

\author[$\beta$]{Yunhao Tang}
\author[$\beta$]{Zhaohan Daniel Guo}
\author[$\beta$]{Zeyu Zheng}
\author[$\beta$]{Daniele Calandriello}
\author[$\beta$]{R\'emi Munos}
\author[$\beta$]{Mark Rowland}
\author[$\beta$]{Pierre Harvey Richemond}
\author[$\beta$]{Michal Valko}
\author[$\beta$]{Bernardo \'Avila Pires}
\author[$\beta$]{Bilal Piot}

\affil[$\beta$]{Google DeepMind}

\begin{abstract}
    Offline preference optimization allows fine-tuning large models directly from offline data, and has proved effective in recent alignment practices. We propose generalized preference optimization (GPO), a family of offline losses parameterized by a general class of convex functions. GPO enables a unified view over preference optimization, encompassing existing algorithms such as DPO, IPO and SLiC as special cases, while naturally introducing new variants. The GPO framework also sheds light on how offline algorithms enforce regularization, through the design of the convex function 
    that defines the loss. Our analysis and experiments reveal the connections and subtle differences between the offline regularization and the KL divergence regularization intended by the canonical RLHF formulation. In a controlled setting akin to \citet{gao2023scaling}, we also show that different GPO variants achieve similar trade-offs between regularization and performance, though the optimal values of hyper-parameter might differ as predicted by theory. In all, our results present new algorithmic toolkits and empirical insights to alignment practitioners.
\end{abstract}

\begin{document}

\maketitle

\section{Introduction}

Reinforcement learning from human feedback (RLHF) has been a canonical paradigm for aligning powerful AI systems along human values \citep{christiano2017deep,ouyang2022training}, as demonstrated by recent advances in large language models (LLMs) \citep{achiam2023gpt,team2023gemini}. RLHF consists of two steps: reward modeling, which trains a reward model $r_\phi$ to capture human preferences from a dataset of pairwise comparison; and regularized policy optimization, which aligns the AI systems against the learned reward model, more formally as below 
\begin{align*}
    \max_\theta  \underbrace{\mathbb{E}_{y\sim\pi_\theta}\left[r_\phi(y)\right]}_{\text{reward maximization}} - \beta \underbrace{\mathbb{KL}(\pi_\theta,\pi_\text{ref})}_{\text{regularization}}.
\end{align*}

\begin{figure}[t]
    \centering
    \includegraphics[width=0.48\textwidth]{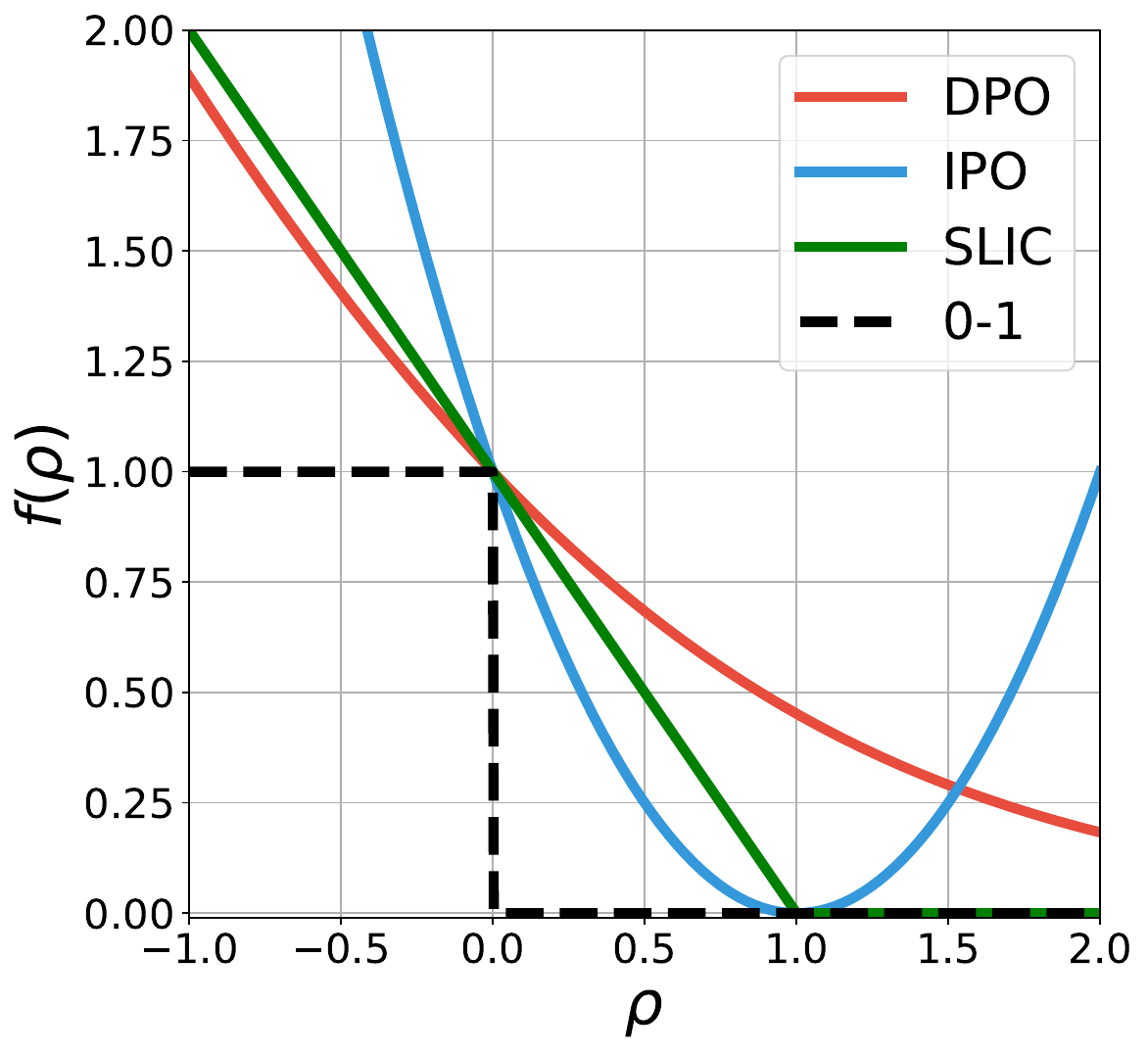}
    \caption{\small{Illustration of offline preference optimization losses $\mathbb{E}_{(y_w,y_l)\sim\mu}\left[f\left(\rho_\theta\right)\right]$ as a function of the difference of log ratio $\rho_\theta=\log \pi_\theta(y_w) / \pi_\text{ref}(y_w) - \log \pi_\theta(y_l)/ \pi_\text{ref}(y_l)$. DPO  applies the (scaled) logistic loss $\frac{1}{\log 2}\log(1+\exp(-\rho_\theta))$, SLiC applies the hinge loss $\max(0,1-\rho_\theta)$, while IPO applies the squared loss $(\rho_\theta-1)^2$. As a result, many popular offline losses can be understood as convex approximations to the 0-1 loss that measures the binary classification accuracy. Any other convex loss alternatives to the above examples provide offline preference optimization losses not in the existing literature, as we show in Table~\ref{table:losses}.}}
    \label{fig:loss}
\end{figure}

Lately, directly aligning AI systems from pairwise comparison datasets has become increasingly common (e.g., \citealp{rafailov2023direct,azar2023general,zhao2023SLiC}), as evidenced by progress in open source models (e.g., \citealp{jiang2024mixtral}). Compared to canonical RL algorithms, such methods are more computationally efficient as they do not require expensive sampling from the models. They also avoid learning reward models altogether, and effectively replace RLHF with a supervised learning problem, which is convenient from various practical perspectives. We refer to such methods as \emph{offline preference optimization}, as they seek to optimize human preferences using offline datasets. Here, \emph{offline} stresses the fact that such datasets are not generated by interactive data collections from the learned model.

Our first contribution is to provide a unifying view over notable existing offline preference optimization algorithms, such as DPO \citep{rafailov2023direct}, IPO \citep{azar2023general} and SLiC \citep{zhao2023SLiC}. To this end, we propose GPO (Generalized Preference Optimization),
which parameterizes preference optimization losses via a family of convex functions $f$, with DPO, IPO, and SLiC as special cases (see Figure~\ref{fig:loss} for a preview of the instantiations). The central insight to our derivation is that one can interpret the problem of reward modeling as a supervised binary classification problem \citep{hastie2009elements}. The rich literature on supervised binary classification paves the way to unifying existing offline preference optimization algorithms, and naturally introduces new algorithms not yet in the current literature. The GPO formulation also helps better understand the algorithmic trade-offs between different variants, particularly, the strength of regularization, which we further dive into.

With a unifying view over offline preference optimization algorithms, our second contribution is to dive into the regularization mechanism induced by offline losses. We see that the tail behavior of the convex function $f$, governs the effective strength of regularization induced between $\pi_\theta$ and $\pi_\text{ref}$, which offers insight on the choice of hyper-parameters such as $\beta$. We identify the offline regularization, computed based on the offline dataset, and show how it generally differs from the KL divergence intended in the initial formulation. Our analysis and empirical results hint at some challenges to enforcing the KL divergence constraints with offline losses, revealing some of the subtleties of the `equivalence' arguments adopted in prior work to derive offline losses (see also Theorem~\ref{thm:equivalence} for a more general version of the equivalence argument). 

The paper is organized as follows:
\begin{itemize}
    \item In Section~\ref{sec:gpo}, we present GPO, generalized policy optimization, which parameterizes offline preference optimization algorithms through a convex function. This recovers a few popular algorithms as special cases and offers insights to offline alignment algorithms in general.
    \item In Section~\ref{sec:derivation}, we expand on the derivation of reward modeling as a binary classification problem. Our insight allows for connecting a rich literature on supervised classification to the designs of offline alignment, which paves the way to the GPO formulation.
    \item In Section~\ref{sec:regularization}, we dive into how offline preference optimization induces regularization between $\pi_\theta$ and $\pi_\text{ref}$ during optimization. We identify an offline regularization loss, the effective regularization that offline algorithms enforce, and show how it differs from the KL divergence through analysis and experimental study. We also show how the design of $f$ introduces different strength of regularization, and how hyper-parameters should be chosen adaptive to $f$.
    \item In Section~\ref{sec:exp}, we start with a controlled setting akin to \citet{gao2023scaling} and show the regularization vs. performance trade-off for different GPO variants. By varying $\beta$ and learning stages during training, the policy performance initially increases followed by decrease, as predicted by the \emph{Goodhart's law}. We observe similar trade-offs across different GPO variants, though the best hyper-parameter can differ significantly due to different inherent strengths of the regularization, as suggested by theory. In a LLM summarization task, we also confirm similar performance across different GPO variants (up to tuning in $\beta$).
\end{itemize}

\section{A general family of offline preference optimization losses} \label{sec:gpo}

In the case of language model alignment, we optimize a policy $\pi_\theta$ that outputs response $y\sim \pi_\theta(\cdot|x)$ given prompt~$x$. Given two responses $y,y'\in\mathcal{Y}$, a human rater provides feedback by picking out the preferred response. This allows relabeling the two responses as $(y_w,y_l)$ corresponding to the win-loss responses. Such pairwise preference data is usually collected offline and can come from a variety of sources in practice, which we denote as a behavior policy $\mu$. Henceforth, when the context is clear we remove the dependency on the prompt $x$ for simplicity.

Importantly, we do not make any assumption on the preference structure $p(y\succ y')$, e.g., it may not come from a Bradley-Terry (BT) model \citep{bradley1952rank}, a common assumption made in prior work \citep{rafailov2023direct}. 
Below, we unify ways to derive various existing offline preference optimization losses for learning from pairwise human feedback.

\subsection{A recipe to derive preference optimization losses}

Assuming access to a reward function $r_\phi$, the regularized policy optimization objective \citep{ouyang2022training} is
\begin{align}
    \max_{\pi_\theta} \mathbb{E}_{y\sim \pi_\theta}\left[r_\phi(y)\right] - \beta \mathbb{KL}\left(\pi_\theta,\pi_\text{ref}\right).\label{eq:po}
\end{align}
To be clear about the KL definition, we have for any two distributions $\pi,\pi'$: $\mathbb{KL}\left(\pi,\pi'\right) \coloneqq \mathbb{E}_{y\sim \pi}\left[\log \frac{\pi(y)}{\pi'(y)}\right]$. The solution to the regularized objective above can be written analytically as $\pi_\theta^\ast(y)\propto \pi_\text{ref}(y)\exp\left(\beta^{-1}r_\phi(y)\right)$.

Given a pair of responses $(y_w,y_l)$, we can train the reward model $r_\phi$ through supervised learning. A convenient class of loss function is defined through the difference $r_\phi(y_w)-r_\phi(y_l)$: we can think of $r_\phi(y_w)-r_\phi(y_l)$ as predicting how likely $y_w$ is preferred to $y_l$. From the discussion above, we see that this difference is equivalent to the log ratio difference of the optimal policy to Eqn~\eqref{eq:po}
\begin{align}
    r_\phi(y_w) - r_\phi(y_l) = \beta \left( \log \frac{\pi_\theta^\ast(y_w)}{\pi_\text{ref}(y_w)} - \log \frac{\pi_\theta^\ast(y_l)}{\pi_\text{ref}(y_l)}\right).\label{eq:reward-logit-equivalence}
\end{align}
Hence intuitively, any loss defined through the reward difference $r_\phi(y_w) - r_\phi(y_l)$
can introduce a loss over $\pi_\theta$. 

A central insight of this work is framing reward learning as a supervised binary classification problem. We leave a more detailed derivation to Section~\ref{sec:derivation}, which provides additional insights. Letting $f:\mathbb{R}\rightarrow\mathbb{R}$ be a scalar function, in general the reward learning loss (to be \emph{minimized}) can be written as
\begin{align}
    \mathbb{E}_{(y_w,y_l)\sim \mu}\left[f\left(r_\phi(y_w)-r_\phi(y_l)\right)\right].\label{eq:rm}
\end{align}

Before moving on, note that the difference $r_\phi(y_w)-r_\phi(y_l)$ is reminiscent of the BT model assumption. However, we argue that it is more sensible to relate this parametric form to the fact that the RLHF formulation (Eqn~\ref{eq:po}) is a maximization problem, and hence imply that each response can be characterized as a single scalar $r_\phi(y)$.  We provide a more detailed discussion in Section~\ref{sec:derivation}.

Many existing offline preference optimization losses can be cast in this general form by replacing the reward difference by the log ratio difference,
\begin{align}
\mathbb{E}_{(y_w,y_l)\sim \mu}\left[f\left(\beta\cdot\left(\log\frac{\pi_\theta(y_w)}{\pi_\text{ref}(y_w)} - \log\frac{\pi_\theta(y_l)}{\pi_\text{ref}(y_l)}\right)\right)\right].\label{eq:offline-loss}
\end{align}

Henceforth, we denote the log ratio difference as $\rho_\theta\coloneqq \log\frac{\pi_\theta(y_w)}{\pi_\text{ref}(y_w)} - \log\frac{\pi_\theta(y_l)}{\pi_\text{ref}(y_l)}$ and the above loss can be rewritten as $\mathbb{E}_{(y_w,y_l)\sim\mu}\left[f\left(\beta\rho_\theta\right)\right]$. A general recipe to derive offline preference optimization losses is to start with a supervised learning loss function $f$ for reward learning, and replace the reward difference by $\rho_\theta$ (see, e.g., \citealp{hastie2009elements} for a nice overview of such loss functions). We can identify the specific functions $f$ for the most common choices; see illustrations of the losses in Figure~\ref{fig:loss} with $\beta=1$.
\begin{itemize}
    \item DPO: $f(\beta\rho_\theta) = -\log \sigma(\beta\rho_\theta)$ with $\sigma$ being the sigmoid function, applies the logistic loss \citep{hastie2009elements}. The loss can also be written as  $\log(1+\exp(-\beta\rho_\theta))$. 
    \item IPO: $f(\beta\rho_\theta) = \left(\beta\rho_\theta-1\right)^2$, the squared function \citep{rosasco2004loss}, can be understood as applying linear regression to the probability that $y_w$ is preferred \citep{hastie2009elements}.
    \item SLiC: $f(\beta\rho_\theta)=\max(0,1-\beta\rho_\theta)$ is the hinge loss function, stemming from the max-margin (support vector machine) paradigm \citep{boser1992training,cortes1995support}. The original SliC algorithm \citep{zhao2023SLiC} also includes a supervised learning component, which we do not discuss here.
\end{itemize}

\begin{table*}
\centering
    \caption{Side-by-side correspondence between existing offline preference optimization losses and convex supervised learning losses. Among a rich variety of convex supervised learning losses developed in the literature, logistic log loss \citep{hastie2009elements}, hinge loss \citep{cortes1995support} and squared loss \citep{rosasco2004loss} have offline preference optimization algorithmic counterparts. Other notable losses, such as the exponential loss \citep{freund1995desicion}, truncated quadratic loss \citep{bartlett2006convexity} and Savage loss \citep{masnadi2008design} can form novel offline preference optimization algorithms. 
    \newline}
\begin{small}
\begin{sc}
 \begin{tabular}{l|c|c}\toprule[1.5pt]
 Supervised learning losses & $f(\beta\rho_\theta)$ & Offline preference  optimization
\\\midrule
Logistic log loss & \bf $\log\left(1+\exp(-\beta\rho_\theta)\right)$ & DPO \citep{rafailov2023direct} \\
Hinge loss & \bf $\max\left(0,1-\beta\rho_\theta\right)$ & SLiC \citep{zhao2023SLiC} \\
squared loss & \bf $\left(\beta\rho_\theta-1\right)^2$ & IPO \citep{azar2023general} \\
Exponential loss & \bf $\exp(-\beta\rho_\theta)$ & N/A \\
Truncated quadratic loss & \bf $\left(\max(0,1-\beta\rho_\theta)\right)^2$ & N/A \\
Savage loss & \bf $1/(1+\exp(\beta\rho_\theta))^2$ & N/A \\
\bottomrule
\end{tabular}
\end{sc}
\end{small}
\vskip -0.1in
\label{table:losses}
\end{table*}

\subsection{GPO: A generalized family of offline preference optimization algorithms}

Building on the discussion above,
in general, any properly defined supervised learning loss $f$ for reward modeling can translate into a preference optimization objective $\mathbb{E}_{(y_w,y_l)\sim\mu}\left[f(\beta\rho_\theta)\right]$. We provide a table of a few notable supervised learning losses developed in the decades-old literature, each loss mapping into an offline preference optimization algorithm.

As discussed above, some of them have already translated into existing methods.
We note a few examples without offline preference optimization counterparts:
\begin{itemize}
    \item Exponential loss: $f(\beta\rho_\theta)=\exp(-\beta\rho_\theta)$, the loss function for the AdaBoost algorithm \citep{freund1995desicion}.
    \item Truncated quadratic: $f(\beta\rho_\theta)=\left(\max\left(0,1-\beta\rho_\theta\right)\right)^2$ \citep{bartlett2006convexity}, a truncated variant of the squared loss, is also a smooth approximation to the hinge loss.
    \item Savage loss: $f(\beta\rho_\theta)=1/\left(1+\exp(\beta\rho_\theta)\right)^2$ \citep{masnadi2008design} which have proved robust to outliers in data and found applications in boosting algorithms.
\end{itemize}
\citet{rosasco2004loss,bartlett2006convexity} give a more exhaustive list of convex supervised learning losses and their discussions.

\begin{figure}[t]
    \centering
    \includegraphics[width=0.48\textwidth]{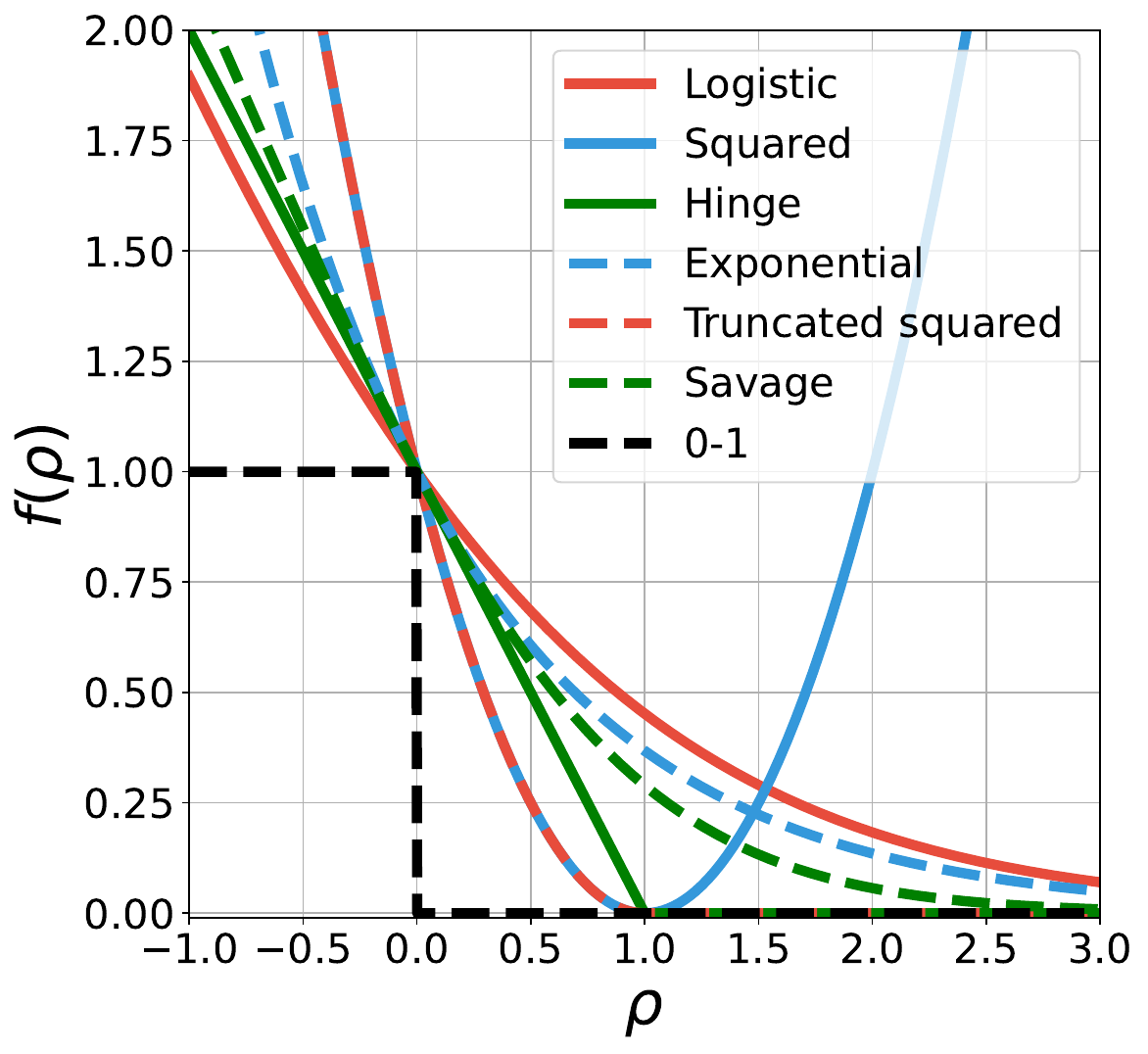}
    \caption{\small{Illustration of notable examples of binary classification loss functions, including both examples (logistic, squared and hinge) that have led to existing offline preference optimization algorithms, as well as others (exponential, truncated squared, Savage) that produce novel losses.}}
    \label{fig:loss-all}
\end{figure}

A key motivating argument for the offline preference optimization algorithms \citep{rafailov2023direct,azar2023general,zhao2023SLiC} is that minimizing the offline losses for the policy $\pi_\theta$ is equivalent to obtaining the optimal regularized policy against a loss minimizing reward model. We can extend the conclusion to this general family of offline preference optimization algorithms.

\begin{restatable}{theorem}{thmequivalence}\label{thm:equivalence}
    (\textbf{Equivalence of optimal solutions}) Let $\pi_\theta^\ast$ be the global minimizer of the offline preference optimization loss in Eqn~\eqref{eq:offline-loss}. $\pi_\theta^\ast$ is the same as the optimal regularized policy (according to Eqn~\eqref{eq:po}) for a reward function that globally minimizes the loss Eqn~\eqref{eq:rm}.
\end{restatable}

\section{Reward modeling viewed as a binary classification problem} \label{sec:derivation}

Here, we take a step back and dive into the derivation that converts reward modeling into a supervised binary classification problem. We provide a brief background on the basic setup, and how it relates to reward modeling (see, e.g., \citealp{hastie2009elements} for a more comprehensive introduction).

In binary classification, given a pair of feature and label $(z, l)$ with $z\in\mathbb{R}^k$ and $l \in\{-1,1\}$, the aim is to predict $\hat{\ell}(z)\in\mathbb{R}$ as a function of the feature, and use $\text{sign}\left(\hat{\ell}(z)\right)$ as the classifier, in the hope that it can match the ground truth label $y$. The classification accuracy can be written as $\frac{1}{2}\mathbb{E}\left[\text{sign}\left(\hat{\ell}(z)\cdot l\right)\right] + \frac{1}{2} \in [0,1]$ and an equivalent loss function is
\begin{align}
    \mathbb{E}\left[1-\text{sign}\left(\hat{\ell}(z)\cdot \ell\right)\right].\label{eq:0-1}
\end{align}
The above loss, known as the 0-1 loss (see the dotted dark curve in Figure~\ref{fig:loss}) is non-convex. Instead of directly optimizing it, we can take smooth convex functions $f:\mathbb{R}\rightarrow\mathbb{R}$ and approximate the loss as
\begin{align*}
    \mathbb{E}\left[f\left(\hat{\ell}(z)\cdot \ell\right)\right].
\end{align*}

Taking this back to the case of reward modeling, given a pair of responses $(y_1,y_2)$, we construct a sample for binary classification by setting the label $\ell=1$ if $y_1\succ y_2$ and $\ell=-1$ otherwise. 

Thinking of $(y_1,y_2)$ as the feature from which to make prediction, in general the prediction would be a bi-variate function $\hat{\ell}(y_1,y_2)$ that can depend on both $y_1$ and $y_2$ in an arbitrary form. For a pointwise reward model that depends on a single response $r_\phi:\mathcal{Y}\rightarrow\mathbb{R}$, an intuitive parameterization would be to take the difference of two rewards $\hat{\ell}(y_1,y_2)=r_\phi(y_1)-r_\phi(y_2)$. The corresponding binary classification loss is
\begin{align*}
    \mathbb{E}_{y_1\sim\mu,y_2\sim\mu}\left[\mathbb{I}\left[y_1\succ y_2\right] f\left(r_\phi(y_1)-r_\phi(y_2)\right)\right] + \mathbb{E}_{y_1\sim\mu,y_2\sim\mu}\left[\mathbb{I}\left[y_2\succ y_1\right] f\left(r_\phi(y_2)-r_\phi(y_1)\right)\right].
\end{align*}
Equivalently, we can write the loss as in Eqn~\eqref{eq:rm}
\begin{align*}
    \mathbb{E}_{(y_w,y_l)\sim \mu}\left[f\left(r_\phi(y_w)-r_\phi(y_l)\right)\right].
\end{align*}

The above result offers a number of interesting implications, which we expand on in the next section.

\subsection{Characterizing what the reward model learns}

Drawing inspiration from the supervised learning literature, we can reason about properties of the reward models obtained by minimizing the convex loss function $f$. This can translate into effects on the downstream optimized policies due to the equivalence in Eqn~\eqref{eq:reward-logit-equivalence}. Some discussions are in order below.

\paragraph{The Bradley-Terry assumption and analytic forms of reward models.} 

As alluded to earlier, the design of the reward modeling loss as a function of the reward difference $r_\phi(y_w)-r_\phi(y_l)$ should be interpreted as a result of the reward maximization formulation of RLHF. Implicitly, the maximization formulation assumes that there is a total order on all the responses (i.e., they can be ranked in a monotonic order), which intuitively is captured by the BT assumption to a large extent. Meanwhile when there is no total order, the formulation Eqn~\eqref{eq:po} would not be perfect, and one might need to resort to alternative solution concepts such as Nash equilibrium \citep{munos2023nash,swamy2024minimaximalist}. 

In general, one should train a pairwise preference model $\hat{\ell}(y_1,y_2)=r_\phi(y_1,y_2)$ rather than pointwise reward models, for which there could be characterizations on the properties of the learned model that we discuss below. For pointwise models
the analytic forms are only available in a few special cases drawn from prior work. We discuss two notable examples: (1) the logistic loss, under the assumption that the ground truth preference satisfies a BT model $p(y_1\succ y_2)=\sigma\left(r^\ast(y_1)-r^\ast(y_2)\right)$, then the optimal reward obtained by minimizing Eqn~\eqref{eq:rm} is a constant shift from $r^\ast$ \citep{rafailov2023direct}; (2) For the squared loss, where the optimal reward is a constant away from $p(y\succ \mu)=\mathbb{E}_{y'\sim\mu}\left[p(y\succ y')\right]$ without further assumptions on the ground truth preference. For interested readers, note that the discussion here also provides an alternative way to derive the IPO algorithm distinct from the original derivation in \citet{azar2023general}.

\begin{figure*}[t]
    \centering
    \includegraphics[width=0.95\textwidth]{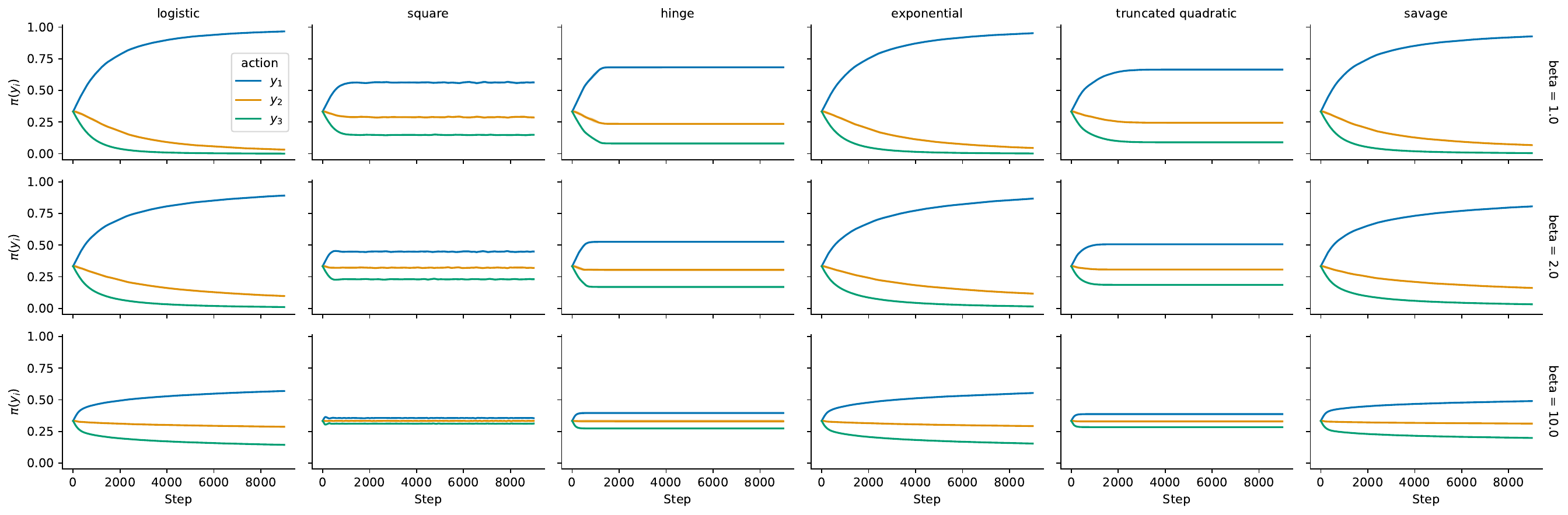}
    \caption{\small{Bandit example \citep{azar2023general} to illustrate the regularization effect of different GPO variants. Convex loss functions with a fast decaying tail or upwards tail (hinge, truncated quadratic and squared loss) will penalize response-level deviations from $\pi_\theta$ to $\pi_\text{ref}$, effectively enforcing a stronger regularization. Other convex losses we exhibit here generally have a slower decaying tail, and will more likely converge to deterministic policies in pathological cases (e.g., deterministic preference).}}
\end{figure*}

\paragraph{A case study of logistic loss vs. hinge loss.} 

Considering the special case when the preferred and non-preferred samples are separable, the hinge loss will find the optimal separating hyperplane that maximizes the margin between the two sets of samples. Drawing inspiration from the classic comparison between logistic regression and support vector machine \citep{hastie2009elements}, we note that the logistic loss will find a similar decision boundary (i.e., sign of the prediction), but it will try to increase the magnitude of the prediction $\hat{\ell}(y_w,y_l)$ to infinity. Such behavior is alluded to in the IPO work \citep{azar2023general} as a failure case of DPO. In general, convex loss functions with a fast-decaying tail (e.g., hinge loss for SLiC) or upwards tail (e.g., squared loss for IPO) will alleviate such issues. In Section~\ref{sec:regularization}, we will illustrate such insights in combination with policy optimization.

\paragraph{General requirement on the convex function $f$.} Not all convex functions $f$ can lead to valid loss functions for binary classification. For our study, we further assume $f'(0)<0$, i.e., $f$ locally decreases at $\rho_\theta=0$. This means that the minimizer of $f$ is obtained at some $\rho_\theta>0$, and intuitively would push the reward difference $r_\phi(y_w)-r_\phi(y_l)$ in the right direction. Intriguingly, this condition is related to Bayes consistency \citep{rosasco2004loss,bartlett2006convexity}, i.e., under which condition can the prediction function $\hat{\ell}(y_1,y_2)$ recover the same sign as the preference probability $\text{sign}\left(2p(y_1\succ y_2)-1\right)$. We provide discussions for interested readers in Appendix~\ref{appendix:bayes}.

\section{Understanding regularization in offline preference optimization} \label{sec:regularization}

In this section, we seek to gain a better understanding of the regularization implicitly enforced by the offline preference optimization algorithms.

Though in general it is challenging to characterize the full learning dynamics of the offline algorithms, we provide analysis from a few angles, which might shed light on how the regularization works. Recall that in the RLHF formulation (Eqn~\ref{eq:po}), the KL regularization is a key element; we will see its connections to the offline regularization.

\subsection{How do offline losses enforce regularization}

As hinted at before, henceforth will we consider the class of convex loss functions that are locally decreasing at $\rho_\theta=0$, i.e., $f'(0)<0$. All the examples in Table~\ref{table:losses} satisfy this property. 

To shed light on how such loss functions entail preference optimization while enforcing regularizers, we consider the Taylor expansion around $\rho_\theta=0$, which is a valid approximation when $\rho_\theta$ is small, i.e., $\pi_\theta$ does not deviate much from $\pi_\text{ref}$.
\begin{align*}
   \underbrace{\mathbb{E}_{(y_w,y_l)\sim \mu}\left[f(\beta\rho_\theta)\right]}_{\text{offline loss}} \approx f(0) + \underbrace{f'(0)\beta\cdot\mathbb{E}_{(y_w,y_l)\sim \mu}\left[\rho_\theta\right]}_{\text{preference optimization}} + \underbrace{\frac{f''(0)\beta^2}{2}\cdot\mathbb{E}_{(y_w,y_l)\sim \mu}\left[\rho_\theta^2\right]}_{\text{offline regularization}}, 
\end{align*}
The expansion implies that when the approximation is valid, the offline algorithms all resemble the case where $f$ is the squared loss (i.e., the IPO loss \citep{azar2023general}). We provide more discussion in Appendix~\ref{appendix:proof}. Minimizing the Taylor-expanded objective achieves two purposes: preference optimization and regularization towards the reference policy. Indeed,  minimizing the first-order term 
\begin{align*}
    f'(0)\beta\cdot\mathbb{E}_{(y_w,y_l)\sim \mu}[\rho_\theta]
\end{align*}
encourages $\pi_\theta$ to place more weight on the preferred response $y_w$ over $y_l$, hence maximizing pairwise human preference. 

To see the effect of the regularization, when $f''(0)>0$ observe that the second-order term 
\begin{align}
    f''(0)\beta^2\cdot \mathbb{E}_{(y_w,y_l)\sim \mu}\left[\frac{1}{2}\rho_\theta^2\right]\label{eq:squared-regularization}
\end{align}
is minimized at $\rho_\theta=0$, in which case $\pi_\theta(y)=\pi_\text{ref}(y)$ for all $y$ in the support of $\mu$. In general, this loss will encourage $\pi_\theta$ to stay close to $\pi_\text{ref}$. We call the above \textbf{$\mu$-weighted squared loss}. Importantly, the global minimizer of the KL divergence between $\pi_\theta$ and $\pi_\text{ref}$ is also a minimizer of the $\mu$-weighted squared loss (i.e., both minimized when $\pi_\theta=\pi_\text{ref}$).

When the approximation is valid, the GPO problem with a regularizer $\beta$ is corresponds to the IPO problem with regularizer $|f''(0)/f'(0)|\cdot\beta$,
and this quantity determines the relative strength of the regularization.
The coefficient $|f''(0)/f'(0)|$ interestingly relates to how convex loss functions are theoretically built-in to be regularized for better generalization \citep{masnadi2015view}. This may inform the design of offline preference optimization algorithms with another theoretical perspective.

\paragraph{Intuition about the full gradient update.} The Taylor expansion is only valid near $\rho_\theta=0$ and except for the special case of squared loss (IPO), drops higher order terms. For example, the expansion does not work natively for SLiC, which employs a non-smooth convex function. Though understanding the full learning dynamics is challenging, we can provide some intuitions about how the full gradient update enforces $\pi_\theta$ to stay close to $\pi_\text{ref}$: consider the gradient update for when $\beta=1$,
\begin{align}
    \theta \leftarrow \theta - \mathbb{E}_{(y_w,y_l)\sim\mu}\left[f'(\rho_\theta)\nabla_\theta \rho_\theta\right].\label{eq:gradient-update}
\end{align}
Starting from $0$, suppose $\rho_\theta$ takes a very high value. This means potentially $\pi_\theta$ places many more weights on certain responses than $\pi_\text{ref}$, which is what the KL divergence regularization seeks to prevent. For the offline update, since $f$ is convex, a few cases are possible: case I: $f'(\rho_\theta)<0$ (for logistic, exponential and Savage loss), $\rho_\theta$ will continue to increase but with a vanishing gradient; hence the regularization is still in place. Meanwhile for case II: $f'(\rho_\theta)\leq 0$ (for hinge, smoothed quadratic and squared loss), $\rho_\theta$ will stop updating or be pushed downwards. As a result, in case  II the gradient update explicitly does not allow $\pi_\theta(y)$ to deviate from $\pi_\text{ref}(y)$ for individual responses $y$, effectively enforcing a stronger regularization with a fixed value of $\beta$.

In Figure~\ref{fig:bandit}, we illustrate the effect of strong regularization using the $3$-action bandit example presented in \citep{azar2023general}, where a simple offline dataset with three pairs of examples are used for training softmax parameterized policies: $(y_1,y_2), (y_2,y_3), (y_1,y_3)$. Examples are uniformly sampled from the distribution. Since $y_1$ is the strongest response, we expect the algorithms to assign high weights to $\pi_\theta(y_1)$, causing deviation from $\pi_\text{ref}$ which is uniform.
The example is meant to illustrate the undesirable behavior of DPO, which tends to push up the probability of $y_1$, despite the intended regularization. See Appendix~\ref{appendix:exp} for more details on the setup. 

We generalize their observations by noting that for any given values of $\beta$, case I losses will keep pushing up the probability of a winning action $y_1$, whereas case II losses enforce the constraint much more conservatively, preventing deterministic policies. In practice where preferences over responses are almost never deterministic, we will see that case I losses are also reasonably well behaved.

\begin{figure}[t]
    \centering
    \includegraphics[width=0.48\textwidth]{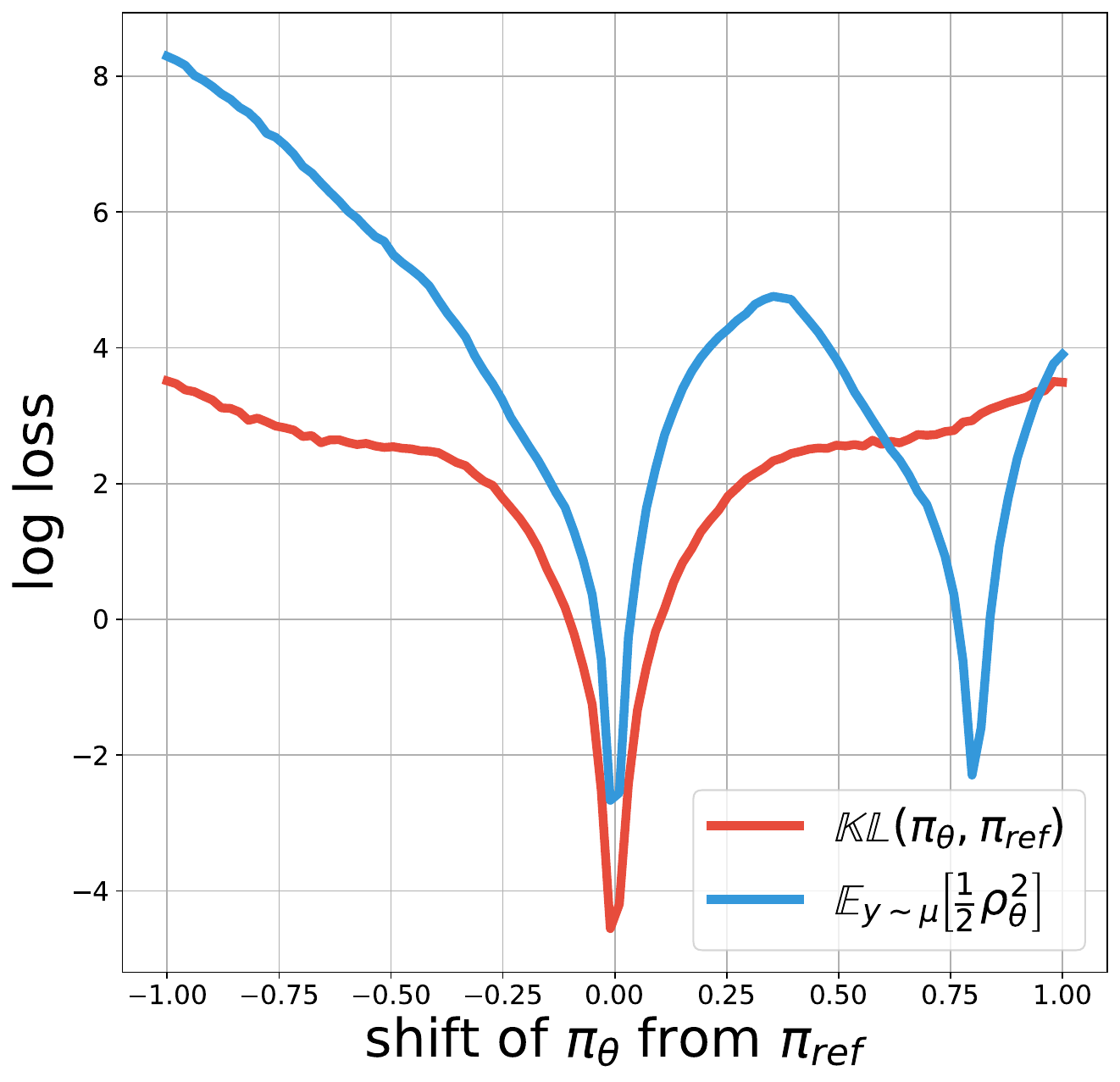}
    \caption{\small{An example of $\mu$-weighted squared loss and KL divergence for mixture of Gaussians. The squared loss has local minimizers different from the KL divergence. This means locally descending on the squared loss may not lead to decreases in the KL divergence, and may not find the global minimizer of the KL divergence. See Appendix~\ref{appendix:exp} for the pdf of $\pi_\text{ref}$ and $\mu$.}}
    \label{fig:gaussians}
\end{figure}

\paragraph{Choosing the right value for $\beta$.} if we understand the tail behavior of the convex function as determining the \emph{natural} regularization strength of the offline algorithm, the hyper-parameter $\beta$ needs to chosen accordingly, if one desires a fixed level of regularization. For example, the logistic loss (i.e., DPO) requires a higher value of $\beta$ to enforce the same level of regularization as the squared loss (i.e., IPO) and the hinge loss (i.e., SLiC), as also exemplified in Figure~\ref{fig:bandit}.

\subsection{Offline regularization vs. KL regularization}

Henceforth we will resort back to the offline regularization: $\mu$-weighted squared loss, and understand its difference against the KL divergence regularization. We start with the gradient of the $\mu$-weighted squared loss 
\begin{align*}
    \mathbb{E}_{y\sim \mu}\left[\nabla_\theta\frac{1}{2}\rho_\theta^2\right]
\end{align*}
which seeks to decrease the squared error that measures the discrepancy between $\pi_\theta$ and $\pi_\text{ref}$, at samples generated by $\mu$. For the KL divergence, we can show that its gradient is equivalent to the $\mu$-weighted squared loss with  $\mu=\pi_\theta$
\begin{align}
\nabla_\theta \mathbb{KL}(\pi_\theta,\pi_\text{ref})= \mathbb{E}_{\color{blue}y\sim \pi_\theta}\left[\nabla_\theta \frac{1}{2}\rho_\theta^2\right].\label{eq:kl-mu}
\end{align}
In other words, we can understand the gradient to the KL divergence as minimizing the discrepancy with \emph{on-policy} samples under $\pi_\theta$, rather than offline samples from $\mu$. We detail the derivation in Appendix~\ref{appendix:proof}; note a highly similar result was also derived in \citep{richter2020vargrad}.

In summary, both losses enforce the squared penalty on samples from $\mu$ vs. online samples from $\pi_\theta$. We can envision cases when the $\mu$-weighted squared loss is being minimized, the KL divergence might not decrease as desired.

\paragraph{A mixture of Gaussians counterexample.} To show the fact that, during minimization of the squared loss, we may not necessarily observe global minimization of the KL divergence, we provide a low-dimensional toy counterexample using mixture of Gaussians. We set up an example where both $\pi_\text{ref}$ and $\mu$ are mixtures of three Gaussians. The optimized policy $\pi_\theta$ is just a constant shift away from $\pi_\text{ref}$ with the shift being parameterized by a trainable parameter $c$. When $c=0$, we have $\pi_\text{ref}=\pi_\theta$ and both the squared loss and KL divergence are minimized to $0$.

In Figure~\ref{fig:gaussians}, we show the KL divergence and the $\mu$-weighted squared loss, both in log scales, as a function of $c\in[-1,1]$. The squared loss has a few minima, with some of them being remote from $c=0$. This means gradient descent on the squared loss may not lead to smaller KL in general, though they are both globally minimized at $\pi_\theta=\pi_\text{ref}$ for $c=0$. See Appendix~\ref{appendix:exp} for the plot of the pdfs of $\mu$ and $\pi$.

This example is meant to illustrate that the arguments used in prior work on offline preference optimization \citep{rafailov2023direct}, which heavily rely on the global minimization of objectives, may not always be true in practice: locally minimizing the $\mu$-weighted squared loss might not lead to decrease in the KL divergence. However, the silver lining is that near $\pi_\theta=\pi_\text{ref}$, the two losses are highly correlated; we will validate the observations on such a low-dimensional example with a language modeling study.

\begin{figure*}[t]
    \centering
    \includegraphics[width=0.95\textwidth]{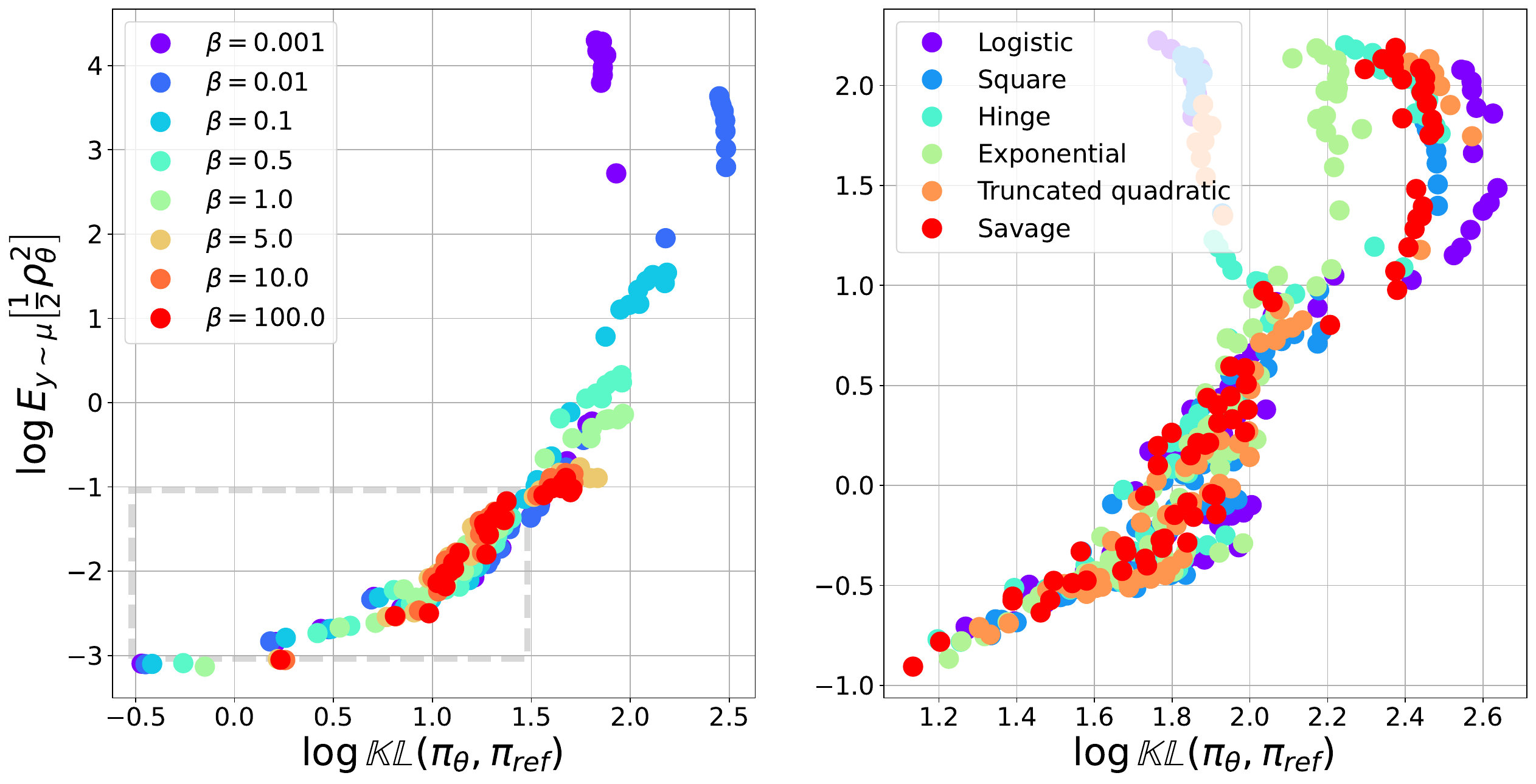}
    \caption{\small{Tracing out KL divergence vs. $\mu$-weighted squared loss during offline preference optimization. (Left) With $f$ being the squared function, we show the trajectories for a range of $\beta$s. Importantly, the initial data point for which $\pi_\theta=\pi_\text{ref}$ is dropped for better visualization, see Appendix~\ref{appendix:exp} for the complete plot. Note that as $\beta$ increases, the algorithm maintains a better constraint on the $\mu$-weighted squared loss, which also induces a constraint on the KL divergence. (Right) We pool over different $\beta$s and show trajectories for different GPO variants. See Appendix~\ref{appendix:exp} for individual plots for each variant. Overall, all algorithmic variants enjoy similar constraint properties, with most variants being slightly more stable than the logistic variant.}}
    \label{fig:gpo-all}
\end{figure*}

\subsection{Analyzing a language modeling example}

In the case of language modeling, where $\pi_\theta,\pi_\text{ref},\mu$ are sequential categorical distributions, we measure the correlation between the KL divergence $\mathbb{KL}\left(\pi_\theta,\pi_\text{ref}\right)$ and the $\mu$-weighted squared loss $\mathbb{E}_{y\sim\mu}\left[\frac{1}{2}\rho_\theta^2\right]$ during offline training. We consider the summarization task similar to \citep{roit2023factually}, where the offline dataset is an open source summarization dataset collected with human feedback labels \citep{stiennon2020learning}. We give more details in Appendix~\ref{appendix:exp}.

For each experiment, we choose a fixed value of regularization $\beta$. Then, we initialize $\pi_\theta$ from $\pi_\text{ref}$ and minimize the offline preference losses over the dataset. As the training progresses, we record sample-based estimates of the KL divergence and $\mu$-weighted squared loss over time, and trace them out in Figure~\ref{fig:gpo-all} left plot for when $f$ is a squared function.  We show both loss functions in the log scale. 

Importantly, we have dropped from the plot the initial data point for which $\pi_\theta=\pi_\text{ref}$ and both losses are zero, otherwise the whole plot will look unbalanced (since $\log 0\approx -\inf$). See the full plot in Appendix~\ref{appendix:exp}. We make a few comments regarding the current plot.

\paragraph{Correlation between the two losses.} There appears to be two phases in Figure~\ref{fig:gpo-all} left plot. When $\beta$ is large, and when the $\mu$-weighted squared loss is maintained at a lower level, we see a better correlation between the two losses. Meanwhile, when $\beta$ is small and the $\mu$-weighted squared loss grows quickly during optimization, its correlation with KL divergence becomes more elusive (see purple and blue data points on the left plot). Such observations echo the mixture of Gaussian examples, where in the vicinity of $\pi_\theta=\pi_\text{ref}$, the two losses have similar trends; the misalignment happens when we deviate too much from the origin.

Though the correlation between the two losses seem to break when $\pi_\theta$ is too far away from $\pi_\text{ref}$, the silver lining is that for offline algorithms, the optimization always starts with the origin $\pi_\theta=\pi_\text{ref}$, and one may expect a better control over the KL divergence through the $\mu$-weighted squared loss.

\paragraph{More variations in KL compared to $\mu$-weighted loss.}
For Figure~\ref{fig:gpo-all} left plot, in the regime where the KL divergence and $\mu$-weighted squared loss are better correlated (areas inside the grey bounding box), we see an order of magnitude more drastic variations in the KL divergence ($10^{-0.5}\rightarrow 10^{1.5}$) than the $\mu$-weighted squared loss ($10^{-1.5} \rightarrow 10^{0.5}$). 

This hints at the challenge of maintaining the KL divergence constraint by controlling the $\mu$-weighted squared loss. Indeed, since the offline preference optimization algorithms directly optimize for the $\mu$-weighted squared loss in the vicinity of the origin $\pi_\theta=\pi_\text{ref}$, even small changes in the $\mu$-weighted squared loss can induce much bigger changes in the KL divergence. This might become a source of instability during optimization. However, the degree to which such instability can be mitigated by other hyper-parameter choices such as learning rate, might vary case-by-case.

\paragraph{Comparison across different GPO variants.} In Figure~\ref{fig:gpo-all} right plot we compare the constraint contours across different GPO variants listed in Table~\ref{table:losses}. For each variant we sweep the $\beta$s but for visualization we pool across results from all $\beta$s, see Appendix~\ref{appendix:exp} for individual plots. 

Overall, different variants follow a similar pattern, with most variants being slightly more robust compared to the logistic loss, which seems to induce slightly bigger variations in the KL divergence compared to other alternatives.

\section{Empirical study of GPO variants} \label{sec:exp}

We now carry out a set of experimental comparison between different GPO algorithms, and to study their empirical behavior and validate theoretical insights.

\begin{figure*}[t]
    \centering
    \includegraphics[width=0.95\textwidth]{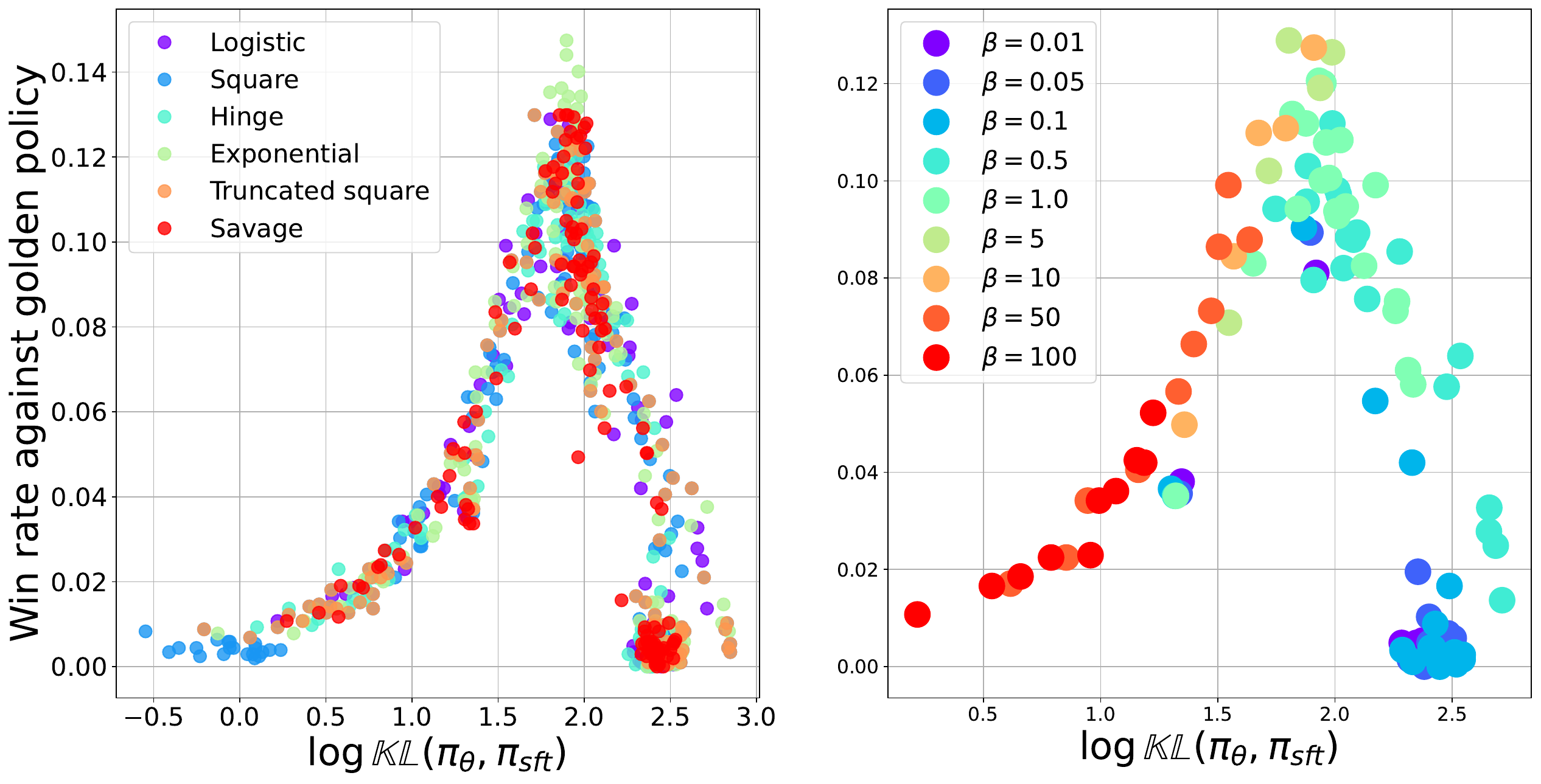}
    \caption{Left: Tracing KL divergence vs. golden win rate performance for different GPO variants. Each data point corresponds to a policy obtained during training with a particular value of $\beta$ and convex function loss. For each loss variant, we pool data points across $\beta$s and different stages of training. Overall, the trade-off curves of GPO variants look similar. Right: Tracing the trade-off for the logistic loss (DPO), grouped according to the regularization coefficient $\beta$. As $\beta$ increases, the regularization effect is larger and during training, and the policies tend to have smaller KL divergence against $\pi_\text{sft}$.}
    \label{fig:goodhart-all}
\end{figure*}

\subsection{Trade-offs between KL divergence and performance}

As the offline alignment optimization progresses, the policy $\pi_\theta$ starts to drift away from the initial anchor policy $\pi_\text{sft}$. When measured in terms of the \emph{ground truth} performance, there is a trade-off between model performance and KL divergence from the initialization. We adopt a synthetic setting similar to \citep{gao2023scaling} to study this trade-off. 

Concretely, we take the summarization task introduced above and train a XXL model (11 billion parameters) as the golden preference model, using similar training setting as \citet{munos2023nash}. This preference model will be used as the \emph{golden} judgement. Since the preference model carries out side by side comparison, we also train a golden policy as the fixed baseline to compare against. We provide more technical details in Appendix~\ref{appendix:exp}. For each fixed convex loss function, we sweep over values of the regularization coefficient $\beta$. For each $\beta$, we train the model for $2\cdot 10^4$ steps with a constant learning rate ($10^{-5}$ and $3\cdot 10^{-5}$). We evaluate checkpoints every $2k$ steps for a total of $20k$ training steps.

In Figure~\ref{fig:goodhart-all} (left), we trace the performance of trained checkpoints over time, plotting their golden evaluation performance against the golden policy. Each dot corresponds to a checkpoint evaluation, for a particular value of $\beta$, learning rate and convex function loss. We group the results by the convex function loss. A few observations are in order: (1) We observe the over-optimization effect compatible with \emph{Goodhart's law} \citet{gao2023scaling}, wherein as the KL divergence increases, the golden performance evaluation first increases and then decreases as a result of over-optimization. The key difference is that \citep{gao2023scaling} is for online RLHF, while our case is offline optimization; (2) For different loss functions, the overall trade-off curves look similar. Concretely, the peak performance is similar and is obtained at a similar level of KL divergence. This suggests that for any choice of the convex loss function, a choice of $\beta$ and training step can lead to a specified level of performance. 

In Figure~\ref{fig:goodhart-all} (right), we break down the trade-off curve with respect to the regularization coefficient $\beta$. We show the case for the logistic loss, though other losses have a similar breakdown (see Appendix~\ref{appendix:exp} for full results). For each $\beta$ (with a unique color), different data points correspond to different stage of training for the same experiment and hence tracing out a trend of KL divergence vs. win rate. We make a few observations: (1) Data points seem to piece together seamlessly at the \emph{soft} boundaries between $\beta$s, this means given a fixed value of $\beta$, one can probably obtain a specified level of KL divergence and win rate performance, by training the policy for a certain number of steps. However, different $\beta$s are not equal: in the case of logistic loss, $\beta \sim 1$ seems to obtain the best overall performance across training, while $\beta=0.01$ can easily train the policy to have large KL divergence, resulting in degraded performance; meanwhile, $\beta=100$ puts a larger constrain the policy near $\pi_\text{sft}$, making it difficult to obtain the best performance across training.

\begin{figure*}
    \centering
    \includegraphics[width=0.95\textwidth]{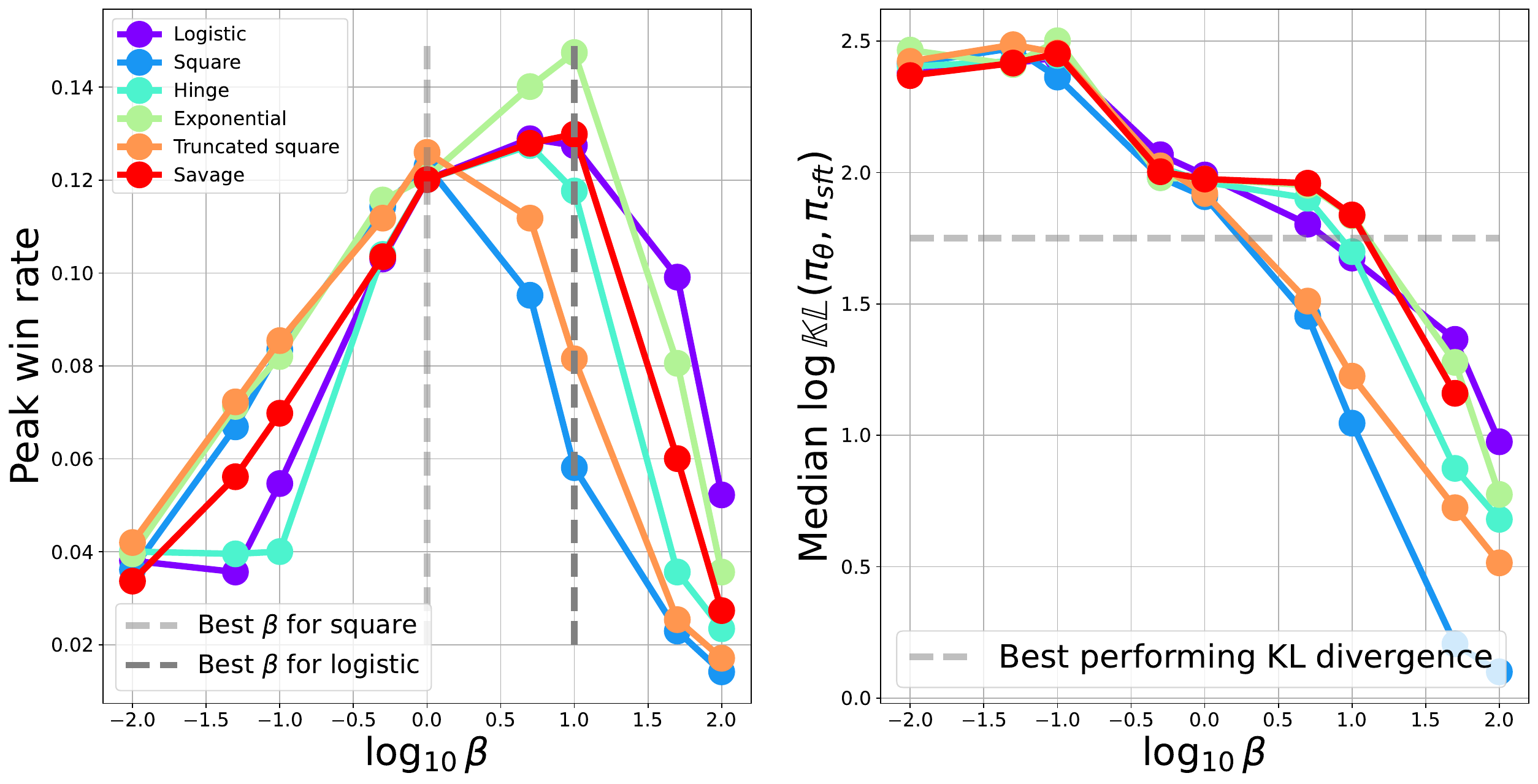}
    \caption{Left: $90\%$-th percentile performance during training for different values of $\beta$s. We use the $90\%$-th percentile as an estimate of the best possible performance under a fixed $\beta$. Different GPO variants seem to peak at different values of $\beta$: noticeably, squared loss and truncated squared loss peak at about $\beta=1$ while others mostly peak at slightly larger values $\beta\sim 10$. Right: Median values of KL divergence during training, as a function of $\beta$ for different GPO variants. When $\beta$ is small, different variants have little distinction; when $\beta$ is large (strong regularization) and fixed, squared and truncated squared loss tend to incur smaller KL divergence compared to other variants.}
    \label{fig:goodhart-betas}
\end{figure*}

\paragraph{Impact of $\beta$.} We now closely investigate the impact that $\beta$ has on the performance and KL regularization dynamics of various GPO variants. Figure~\ref{fig:goodhart-betas} (left) shows the peak performance of various algorithms as a function of $\beta$. As seen from the plot, the peak performance of squared and truncated squared loss is obtained at generally lower $\beta\sim 1$, whereas the peak performance for other variants are obtained at higher $\beta\sim 10$. There is some variations of the peak win rate (e.g., exponential seems to be slightly better than others) but this might not be statistically significant.

While the observation suggests the fact that different algorithms require different values of $\beta$s to perform the best, it can be explained by the fact that different loss functions induce distinct strengths of regularization as a function of $\beta$, as predicted by theory. In Figure~\ref{fig:goodhart-betas} (right) we show the median KL divergence during training as a function of $\beta$, for different convex loss functions. When $\beta$ is small and regularization is weak, there is little distinction between different variants. This is compatible with the results in Figure~\ref{fig:gpo-all}: the offline algorithm enforces regularization through the weighted squared loss, and its correlation with KL divergence is weak when the regularization is small. At large values of $\beta$s, the correlation between offline regularization and KL divergence is much stronger. And indeed, we see squared and truncated squared loss enforce stronger regularization than other variants, with logistic, exponential and Savage being in the same league and hinge loss in the middle.

\subsection{Model-based side by side evaluation}

The synthetic setting has provided many insights into the trade-offs between regularization and policy performance, and how they are modulated by choices of $\beta$ and convex loss functions. We now carry out a final set of experiments on the summarization task, using settings described in prior work \citep{munos2023nash,calandriello2024human}.

We consider the side-by-side comparison metric used by \citet{munos2023nash}, where we compare the checkpoint performance against a fixed opponent $\pi_\text{ref}$. The comparison is made by a prompted PaLM-2 model \citep{anil2023palm} over an evaluation set of $2000$ summary samples. The prompted model judges which response is of higher quality. See Appendix~\ref{appendix:exp} for evaluation details.

Examining the performance across $\beta$s, we see that when $\beta$ is small, the optimization tends to be more effective, achieving the best performance at about $\beta\in[0.1,1]$ across the board, with similar peak performance. The performance experiences a bigger drop when $\beta$ becomes large. When making pairwise comparison across different GPO variants, we see that their performance is generally on par with one another; choosing the right $\beta$ appears more critical. Due to space limits, we present these comparisons in Appendix~\ref{appendix:exp}.

\section{Discussions and conclusion}

We have presented GPO, a generalized approach to deriving offline preference optimization losses for LLM alignment. GPO presents a continuous spectrum of loss functions, encompassing DPO, IPO and SLiC as special instances. By deriving GPO through the rich literature on binary classification, we have presented a more unified way to reason about the strength of regularization and what the optimized policy seeks to capture.

We have shown the connections between the offline regularization and the KL regularization, which the RLHF formulation seeks to enforce. The two types of regularization are different in general. However, optimizing from the origin, we see empirical evidence that the two losses are correlated, alluding to the fact that enforcing KL divergence through offline optimization is possible though maybe more challenging. 

We have also showed the regularization vs. performance trade-off between different GPO variants. Overall, the regularization vs. performance trade-off is similar for different algorithms. As predicted by theory, different convex loss variants induce inherently distinct strengths for regularization, which impacts the optimal value of $\beta$ for each algorithm (i.e., squared loss needs a smaller $\beta$ than logistic loss).

Our results have a number of limitations and provide avenues for future work. Our framework is based on the reward maximization formulation of RLHF, and hence still encounters theoretical issues when the ground truth preference structure is complex. A future direction would be to connect GPO with alternative solution concepts for alignment such as Nash equilibrium \citep{munos2023nash}. Our framework also only deals with offline losses with a contrastive form, and does not handle supervised learning based losses \citep{zhao2023SLiC}.

\paragraph{Acknowledgements.} We thank Ivo Danihelka for providing very valuable feedback to an earlier draft of the paper. We are thankful to the Google DeepMind teams that build the infrastructure which facilitates the research in this work.

\bibliographystyle{plainnat}
\bibliography{main}

\newpage
\appendix

\section{Experiment details and additional results}\label{appendix:exp}
We provide further details and additional results on experiments across the paper.

\subsection{Bandit experiment}

To illustrate the regularization properties of various GPO variants, we have employed the bandit experiment introduced in \citep{azar2023general}. We consider a $3$-action bandit problem where the dataset consists of three possibilities 
\begin{align*}
    (y_1,y_2), (y_2,y_3), (y_1,y_3).
\end{align*}
Sampling from the offline dataset consists in uniformly sampling from the pairs. We then train softmax parameterized policies with exactly the same setup as \citep{azar2023general}. Note that with logistic, exponential and Savage loss, because the tail does not vanish fast enough, the policy converges to the greedy action $y_1$ even with regularization at $\beta=1$. While for the other three losses, thanks to stronger regularization, $\pi_\theta(y_1)$ maintains closer distance to $\pi_\text{ref}(y_1)$.

\subsection{A mixture of Gaussian counterexample}

We find the counterexample by parameterizing all related distributions as mixtures of Gaussians with $3$ modes. It is not difficult to construct numerical counterexamples as shown in the paper, with  $\leq 5$ simulations.

The offline distribution $\mu$ is parameterized as $\mu=\frac{1}{3}\mathcal{N}(u_1,0.05^2)+\frac{1}{3}\mathcal{N}(u_2,0.05^2)+\frac{1}{3}\mathcal{N}(u_3,0.05^2)$ where $u_1,u_2,u_3$ are i.i.d. uniform between $-1$ and $1$. The reference policy is fixed as $\pi_\text{ref}=\frac{3}{10}\mathcal{N}(-0.8,0.1^2)+\frac{4}{10}\mathcal{N}(0,0.1^2)+\frac{3}{10}\mathcal{N}(0.8,0.1^2)$. The optimized policy $\pi_\theta$ is a constant shift away from $\pi_\text{ref}$. The particular choice of the parameters are fairly ad-hoc and other choices of hyper-parameters should lead to clear counterexamples as well. Since for mixtures of Gaussians, both the KL divergence and the $\mu$-weighted squared loss do not yield analytic forms. Instead, we draw $2000$ samples to estimate both losses as unbiased estimates.

Figure~\ref{fig:gaussians-all} (left) shows the probability density function (pdf) for $\pi_\text{ref}$ and $\mu$, in the counterexample that we presented.

\begin{figure*}
    \centering
    \includegraphics[width=0.95\textwidth]{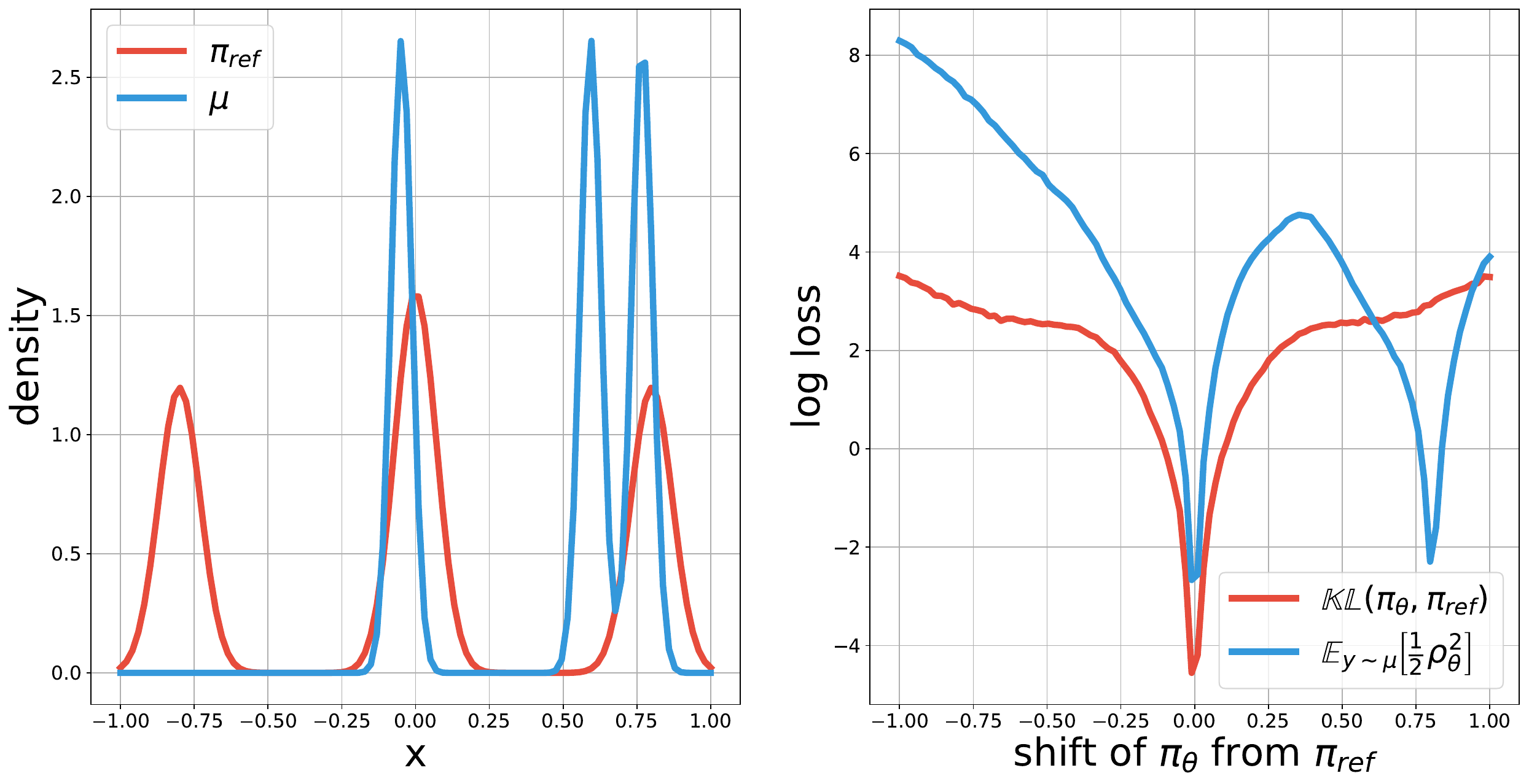}
    \caption{Full results for the mixture of Gaussian counterexample. (Left) The probability density function (pdf) for $\mu$ and $\pi_\text{ref}$, both are designed to be mixtures of Gaussian with $3$ modes; (Right) The same plot as Figure~\ref{fig:gaussians}.}
    \label{fig:gaussians-all}
\end{figure*}

\subsection{Language modeling experiments}

We consider the summarization task similar to \citep{roit2023factually}, where the offline dataset is an open source summarization dataset collected with human feedback labels \citep{stiennon2020learning}. The base model is T5X \citep{roberts2023scaling}, a family of LLMs based on encoder-decoder transformer architecture. Throughout, we train large-sized models with $700M$ parameters. During training, we apply a constant learning rate of $10^{-5}$ with batch size $b=32$. We use the Adafactor optimizer \citep{shazeer2018adafactor} with a decay rate of $0.8$. Each model is trained for $2\times 10^5$ steps in total.

The evaluation follows from \citep{munos2023nash} where they consider the side-by-side comparison metric between two models. A default baseline model is the supervised fine-tuned baseline $\pi_\text{ref}$. The comparison is made by a prompted PALM-2 model \citep{anil2023palm}, where the model judges which response is of higher quality. The evaluation set consists of $2000$ examples, each containing a paragraph to summarize. The prompted model is given the paragraph, as well as the two summaries generated by the two compared models, to deliver a final verdict.

\paragraph{Tracing KL divergence and $\mu$-weighted squared loss.} For each experiment (with a fixed convex function $f$ and fixed $\beta$), we evaluate intermittently the $\mu$-weighted squared loss on the learner and the KL divergence on the evaluator. The evaluator is carried out every $2000$ steps where we train for a total of $20000$ steps. We also evaluate every $200$ steps for the first $2000$ steps since the initial stage during training presents the most salient changes in the $\mu$-weighted squared losses.

The tracing plot for individual GPO variant is shown in Figure~\ref{fig:gpo-individuals} for better visualization.

\subsection{Trade-off between performance and KL divergence}

All experiments are carried out with T5X models \citep{raffel2020exploring} with the T5X data and compute framework \citep{roberts2023scaling}. To create a synthetic setup similar to \citet{gao2023scaling}, we take the summarization dataset and train a golden preference model with the XXL model (11 billion parameters). Then we use the XXL model to relabel the offline dataset, and all offline experiments going forward are carried out with this relabeled dataset. 

Since the preference model requires side by side comparison, we also train a golden policy using online IPO \citep{calandriello2024human} using the golden preference model. This policy is denoted \emph{golden} because it makes use of the golden preference model during training, and should arguably obtain the best possible performance over time. We use this policy as the reference policy during evaluation.

All policies are trained with the Large T5X model (110 million parameters) using offline preference optimization variants outlined in the paper.

\paragraph{Full results on the breakdown of KL divergence vs. win rate.} Figure~\ref{fig:goodhart-all-betas} shows the win rate performance and KL divergence trade-off curves across different algorithmic variants of GPO. For each algorithmic variant, the data points are grouped by the regularization coefficient $\beta$. Overall, different algorithmic variants exhibit trade-off pattern and their dependency on $\beta$ is similar too.

It is worth noting that compatible with results reported in Figure~\ref{fig:goodhart-betas}, all algorithmic variants achieve the peak performance at the same value of KL divergence but with a different value of $\beta$. This is the result of the fact that different loss functions have different natural strength of regularization.

\begin{figure*}
    \centering
    \includegraphics[width=0.95\textwidth]{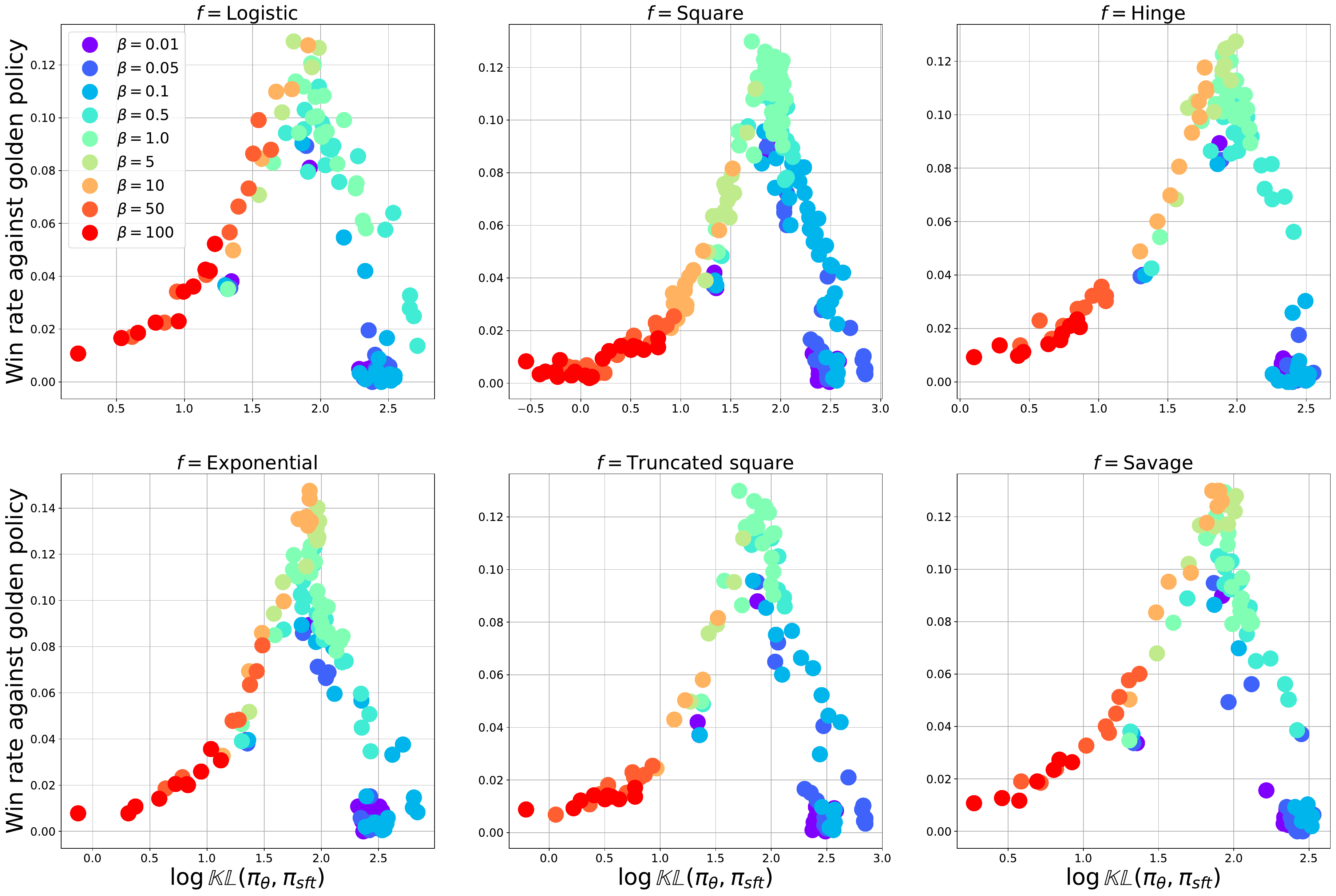}
    \caption{Tracing the trade-off between performance and KL divergence for the various loss functions. For each loss function, the data points are grouped according to the regularization coefficient $\beta$. We see that different algorithmic variants exhibit similar patterns both in terms of the general trade-off curves, as well as the dependency of the curves on $\beta$.}
    \label{fig:goodhart-all-betas}
\end{figure*}

\paragraph{Side by side evaluation.} We subsample $256$ prompts from the training set and generate responses from both the golden policy and the target policy to compare against. We then use the preference model to judge the win rate between the two sets of responses, and average across the subsampled prompt set.

\subsection{Model-based side by side evaluation}

We now discuss experimental results on the summarization task with model-based side by side evaluation. While previous study on the KL divergence vs. win rate trade-off is carried out in a synthetic setting, here we train models with the open sourced summarization dataset \citep{stiennon2020learning} and prompt a PALM-2 model \citep{anil2023palm} for side by side evaluation. We adopt identical evaluation setup as in \citep{munos2023nash} and \citep{calandriello2024human}.

\paragraph{Win rate results.} In Figure~\ref{fig:win-rate}, we show the win rate of various algorithmic variants in a side-by-side comparison against the supervised fine-tuned checkpoint $\pi_\text{ref}$. For two identical models, the win rate should be $0.5$. We observe that the best performance is usually obtained at $\beta\in [0.1,1]$, with similar performance across different $f$s. Interestingly, when $\beta$ becomes too large, the win rate drops more quickly across all methods.

In Figure~\ref{fig:win-rate-sxs}, we show the side by side comparison across GPO variants. For each variant, we take the checkpoint with $\beta=0.1$ since this appears to be a value where all algorithms work reasonably, according to the win rate against the supervised fine-tuned checkpoint. The win rate comparison across GPO variants suggests that they perform mostly similar.

\begin{figure}[t]
    \centering
    \includegraphics[width=0.95\textwidth]{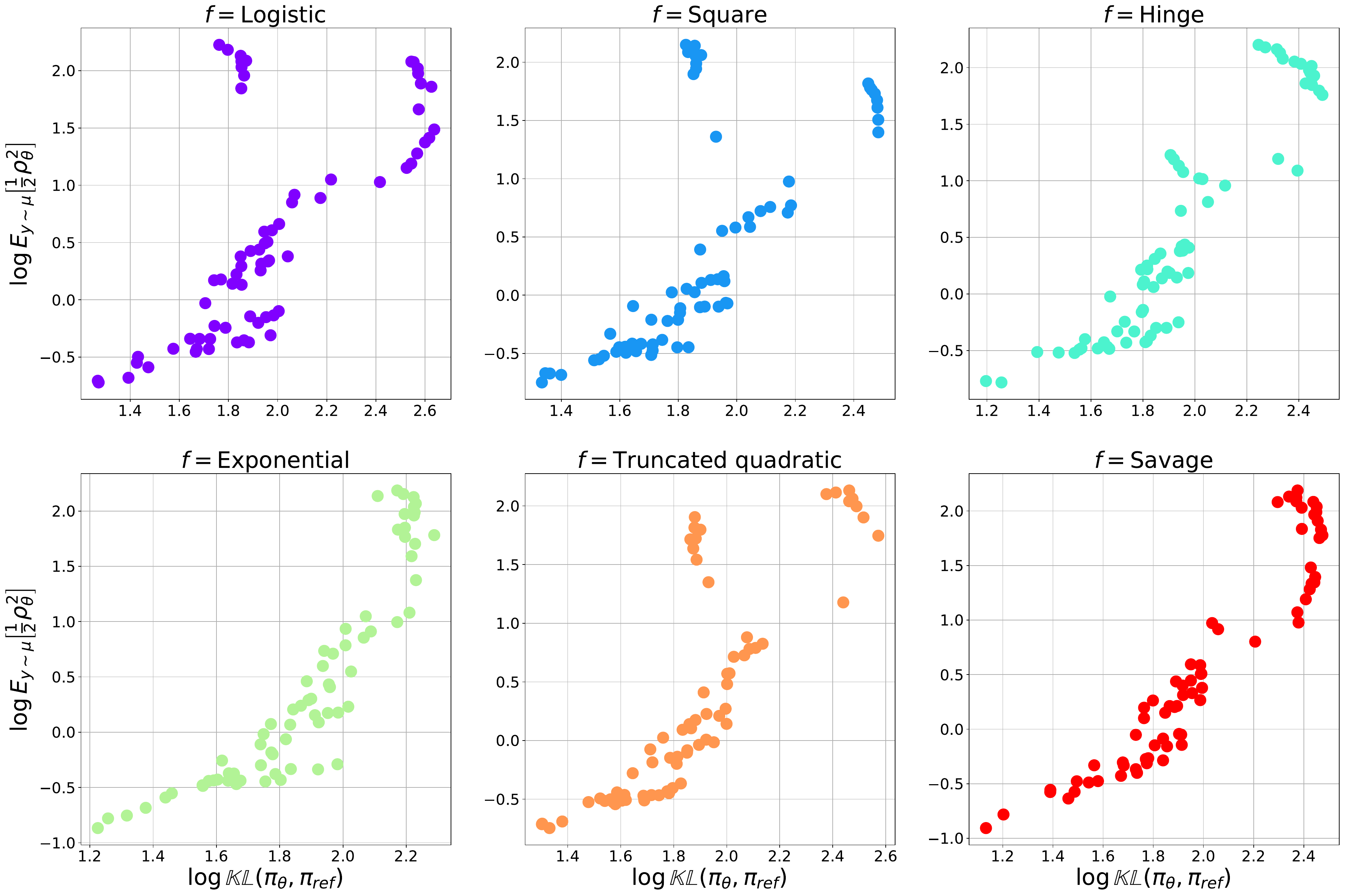}
    \caption{Tracing out KL divergence vs. $\mu$-weighted squared loss during offline preference optimization, for individual GPO variants. This plot separates the data from Figure~\ref{fig:gpo-all} for better visualization.}
    \label{fig:gpo-individuals}
\end{figure}

\begin{figure}[t]
    \centering
    \includegraphics[width=0.48\textwidth]{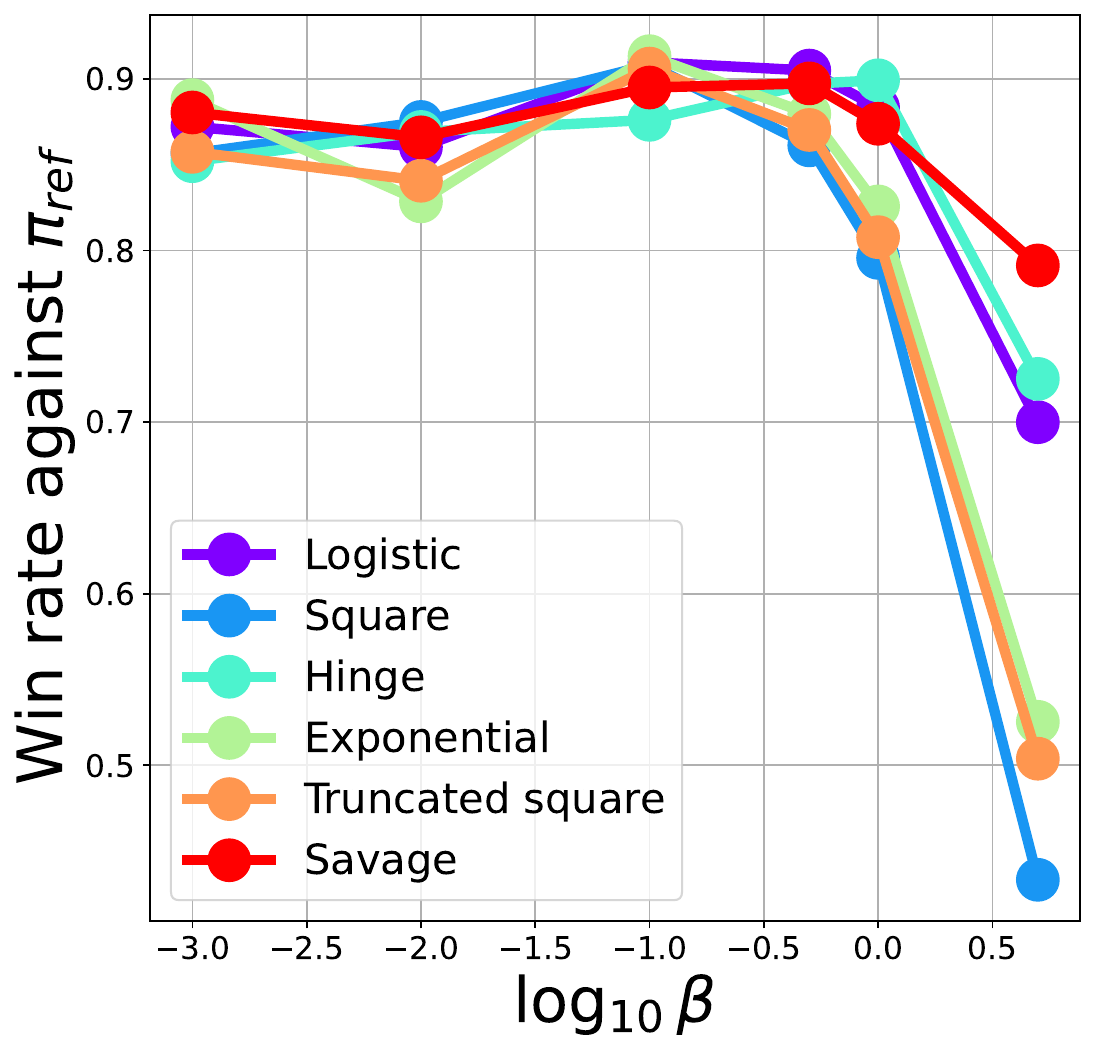}
    \caption{Win rate of various GPO methods against the supervised fine-tuned baseline $\pi_\text{ref}$, as a function of $\beta$. Almost all algorithmic variants obtain the best performance at $\beta\in[0.1,1]$, with similar peak performance.}
    \label{fig:win-rate}
\end{figure}

\begin{figure}[t]
    \centering
    \includegraphics[width=0.48\textwidth]{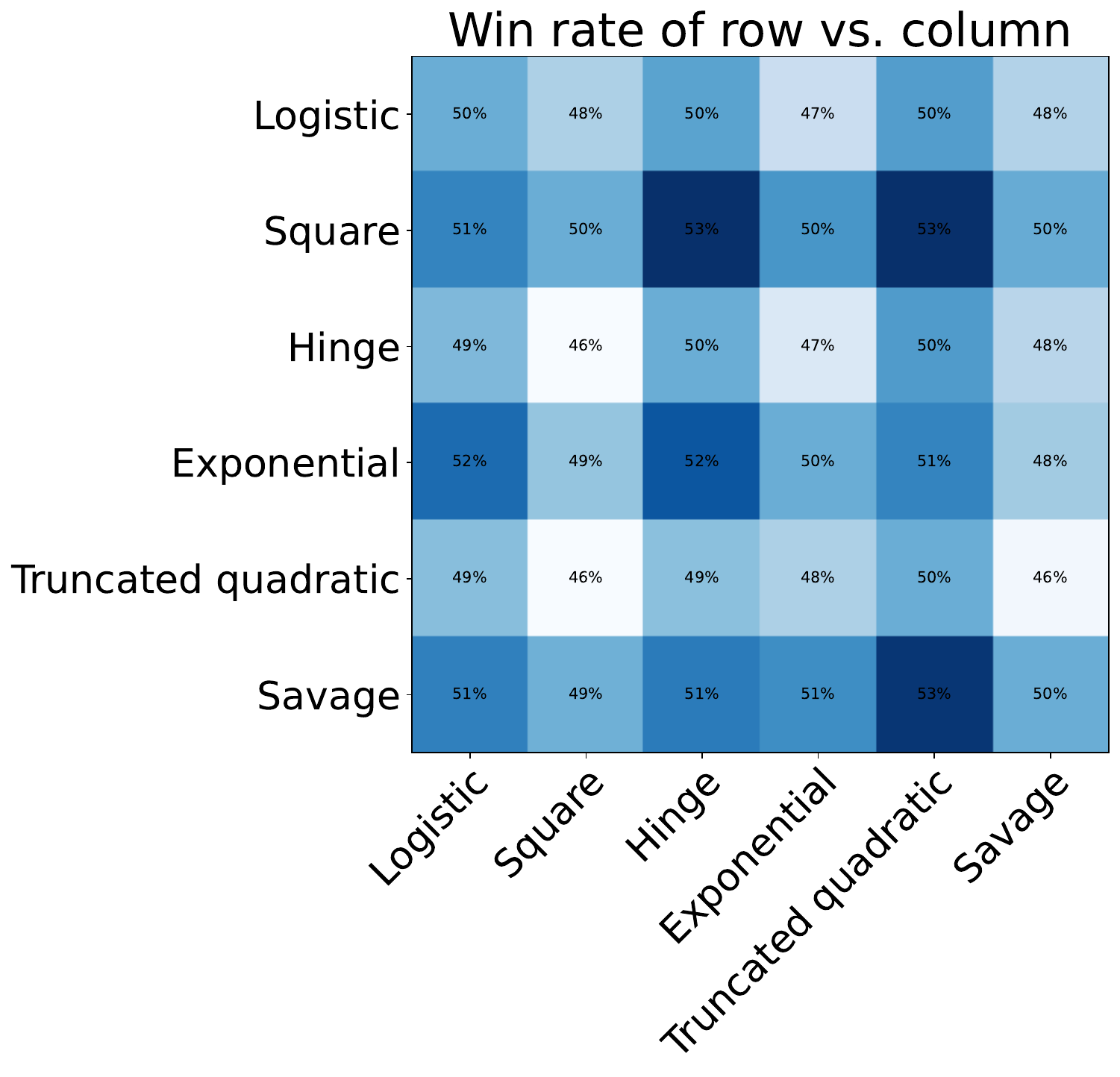}
    \caption{Win rate of various GPO methods against one another. We take all checkpoints at $\beta=0.1$ since this is a value where all variants have reasonable performance. We show the color coded win rates in a matrix.}
    \label{fig:win-rate-sxs}
\end{figure}

\section{Proof and derivations of theoretical results}

We provide more detailed proof to a few important theoretical results in the paper.
\label{appendix:proof}

\thmequivalence*
\begin{proof}
The proof is straightforward. Indeed, note that if we seek to minimize Eqn~\eqref{eq:rm} with $r_\phi$, we can reparameterize the reward function as $r_\phi(y)=\beta \log \frac{\pi_\theta(y)}{\pi_\text{ref}(y)}+z$ with normalizing constant $z$ that depends on $\pi_\theta$. Then if $r_\phi^\ast$ is the global minimizer to Eqn~\eqref{eq:rm}, the corresponding $\pi_\theta(y)\propto\pi_\text{ref}(y)\exp(\beta^{-1}r_\phi^\ast(y))$ must be the global minimizer to Eqn~\eqref{eq:offline-loss}.

\subsection{Derivation of the gradient of KL divergence and $\mu$-weighted squared loss}

By definition we have $\mathbb{KL}\left(\pi_\theta,\pi_\text{ref}\right)=\mathbb{E}_{y\sim \pi_\theta}\left[\log \frac{\pi_\theta(y)}{\pi_\text{ref}(y)}\right]$, its gradient contains two terms
\begin{align*}
    \nabla_\theta \mathbb{KL}\left(\pi_\theta,\pi_\text{ref}\right) = \mathbb{E}_{y\sim \pi_\theta}\left[\log \frac{\pi_\theta(y)}{\pi_\text{ref}(y)}\nabla_\theta \log\pi_\theta(y)\right] + \underbrace{\mathbb{E}_{y\sim \pi_\theta}\left[\nabla_\theta \log \pi_\theta(y)\right]}_{=0}
\end{align*}
The second term vanishes because it is the expectation of a score function with respect to the distribution itself. Meanwhile, for the $\mu$-weighted squared loss, we rewrite the original definition as
\begin{align*}
    \frac{1}{2}\mathbb{E}_{(y_w,y_l)\sim\mu}\left[\rho_\theta^2\right] =  \frac{1}{2}\mathbb{E}_{(y_1,y_2)\sim\mu}\left[\left(\log \frac{\pi_\theta(y_1)}{\pi_\text{ref}(y_1)} - \log \frac{\pi_\theta(y_2)}{\pi_\text{ref}(y_2)}\right)^2\right],
\end{align*}
where the equality is based on the fact that the order of $(y_w,y_l)$ does not impact the expectation. Now, taking the gradient of the above loss with $\mu=\pi_\theta$,
\begin{align*}
    \mathbb{E}_{(y_1,y_2)\sim\pi_\theta}\left[\nabla_\theta \frac{1}{2}\rho_\theta^2\right] &= \mathbb{E}_{(y_1,y_2)\sim\mu}\left[\frac{1}{2}\left(\log \frac{\pi_\theta(y_1)}{\pi_\text{ref}(y_1)} - \log \frac{\pi_\theta(y_2)}{\pi_\text{ref}(y_2)}\right)\left(\nabla_\theta \log \pi_\theta(y_1) - \nabla_\theta \log \pi_\theta(y_2)\right)\right], \\
    &=_{(a) }\frac{1}{2}\mathbb{E}_{(y_1,y_2)\sim\pi_\theta}\left[\log \frac{\pi_\theta(y_1)}{\pi_\text{ref}(y_1)} \nabla_\theta \log \pi_\theta(y_1) + \log \frac{\pi_\theta(y_2)}{\pi_\text{ref}(y_2)} \nabla_\theta \log \pi_\theta(y_2)\right] \\
    &=_{(b)}  \mathbb{E}_{y\sim \pi_\theta}\left[\log \frac{\pi_\theta(y)}{\pi_\text{ref}(y)}\nabla_\theta \log\pi_\theta(y)\right].
\end{align*}
Here, (a) follows from the fact the cross term vanishes because $y_1,y_2$ are independent; (b) follows from the fact that $y_1,y_2$ are identically distributed. This proves the desired equality in Eqn~\eqref{eq:kl-mu}.

\paragraph{Relation to results from \citep{richter2020vargrad}.} A highly related result has been derived in \citep{richter2020vargrad}, relating the gradient of the KL divergence to the gradient of the variance of the log ratio. We provide a simple derivation here. Note that when $\mu=\pi_\theta$, the $\mu$-weighted squared loss indeed evaluates to a variance
\begin{align*}
    \mathbb{E}_{(y_w,y_l)\sim\mu}\left[ \frac{1}{2}\rho_\theta^2\right] = \mathbb{V}\left[\log\frac{\pi_\theta(y)}{\pi_\text{ref}(y)}\right].
\end{align*}
To see this note that if $Y,Y'$ are i.i.d. samples then $\frac{1}{2}\mathbb{E}\left[(Y-Y')^2\right]=\mathbb{V}\left[Y\right]$.
\end{proof}

\subsection{Discussion on Taylor expansions of the GPO losses}

Assume that $f$ is smoothly differentiable and convex, and $f'(0) < 0$, then the GPO problem with the second order Taylor expansion recovers the squared loss with $\beta' = \frac{f''(0)\beta}{|f'(0)|}$. Note that the squared loss is effectively the IPO loss.

To see this, by the second order Taylor approximation to $f$ around $\rho_\theta=0$, we have
\begin{align*}
\mathbb{E}_{(y_w,y_l)\sim \mu}\left[f(\beta\rho_\theta)\right] 
&\approx f(0) + f'(0)\beta\cdot\mathbb{E}_{(y_w,y_l)\sim \mu}\left[\rho_\theta\right] + \frac{f''(0)\beta^2}{2}\cdot\mathbb{E}_{(y_w,y_l)\sim \mu}\left[\rho_\theta^2\right] \\
&= f(0) + \frac{f'(0)^2}{2f''(0)}\left(\frac{f''(0)}{|f'(0)|}\beta\mathbb{E}_{(y_w,y_l)\sim \mu}\left[\rho_\theta\right] - 1\right)^2 - \frac{f'(0)^2}{2f''(0)} \\
&\equiv_{(a)} \mathbb{E}_{(y_w,y_l)\sim \mu}\left[ \left(\frac{f''(0)\beta}{|f'(0)|}\rho_\theta - 1\right)^2\right],
\end{align*}
where for (a) we have rearranged terms and the equivalence is up to constants.
Indeed, we see that the Taylor-expanded GPO loss is equivalent to the IPO loss with $\beta'$ as defined above.

\section{Discussion on Bayes consistency for the learned reward model}\label{appendix:bayes}

Here we provide a brief background on Bayes consistency. Using the notation from Section~\ref{sec:derivation}, we consider binary classification loss of the following form with a convex function $f$
\begin{align*}
    \mathbb{E}\left[f\left(\hat{\ell}(z)\cdot \ell\right)\right]
\end{align*}
where $l\in\{-1,1\}$ is the ground-truth label and $\hat{\ell}(z)$ is the prediction. The Bayes optimal classifier, which minimizes the 0-1 classification error, depends on the probability $p(\ell=1|z)$, which is $\hat{\ell}^\ast(z)=\text{sign}\left(2p(\ell=1|z) - 1\right)$. The Bayes consistency result \citep{rosasco2004loss,bartlett2006convexity} state the following.
\begin{restatable}{theorem}{bayesconsistency}\label{thm:bayes-consistency} (\textbf{Bayes consistency}) Assume $f$ is convex, and continuously differentiable and $f'(0)<0$. Then 
    let $\hat{\ell}(z)$ be the global minimizer to the binary classification loss, then $\text{sign}\left(\hat{\ell}(z)\right)=\hat{\ell}^\ast(z)$.
\end{restatable}
We refer readers to \citet{rosasco2004loss} for the easy-to-follow proof. The high level idea is to show that at the global minimizer, assuming $p(\ell=1|z)>1/2$, we should expect $\hat{\ell}(z)>0$. Intuitively, this should be the case since $f'(0)<0$ and is convex, so the minimizer should be at the right hand side of the origin.

\subsection{Discussion of pairwise preference model}

We now discuss properties of the pairwise preference model, where the prediction $\hat{\ell}(y_1,y_2)$ is parameterized as a general bi-variate function $\hat{\ell}(y_1,y_2)=r_\phi(y_1,y_2)$ of $y_1,y_2$ rather than the difference of two univariate functions $r_\phi(y_1)-r_\phi(y_2)$. We conjecture that some of the results will transfer to pointwise reward models in practice, e.g., when the BT assumption approximately makes sense. Making precise of such approximations is left to future work.

An intuitive requirement for the prediction $\hat{\ell}(y_1,y_2)$ is that it gets the sign of the preference correct, which is defined through $p(y_1\succ y_2)$. More concretely, one might seek the follow property
\begin{align}
    \text{sign}\left(\hat{\ell}(y_1,y_2)\right) = \text{sign}\left(p\left(y_1\succ y_2\right)-1/2\right) \label{eq:correct-sign}
\end{align}

Interestingly, the right-hand side of Eqn~\eqref{eq:correct-sign} corresponds to the Bayes optimal classifier, which minimizes the classification loss in Eqn~\eqref{eq:0-1}. The convex loss functions we consider in this work (e.g., all examples in Table~\ref{table:losses}) all satisfy the property that if $l(y_1,y_2)$ is parameterized as a general preference model (rather than a pointwise reward model, see e.g., \citep{munos2023nash}), then by minimizing the loss we find $\hat{\ell}(y_1,y_2)$ that satisfies Eqn~\eqref{eq:correct-sign}, a result stemming from Bayes consistency \citep{rosasco2004loss,bartlett2006convexity}. 

However, even if different loss functions produce the same sign, the predictions $\hat{\ell}(y_1,y_2)$ can differ drastically depending on $f$. In the main paper we have provided a case study example of logistic loss vs. hinge loss, borrowing inspirations from the study in \citet{hastie2009elements}.

\end{document}